%% file: main.tex
\documentclass[10pt]{article}

\usepackage[accent=blue, seed=1501]{opentslm}

\definecolor{tnNavy} {HTML}{0C1836}   % headings, dark panels
\definecolor{tnBlue} {HTML}{4E6CEA}   % primary accent (Record, TimeNet layer)
\definecolor{tnTeal} {HTML}{12849C}   % Signal / values plane
\definecolor{tnGreen}{HTML}{2A784E}   % Source / outputs / task types
\definecolor{tnAmber}{HTML}{C07E1A}   % Task
\definecolor{tnPink} {HTML}{B0357C}   % Annotation
\defineaccent{timenet}{3556CF}{4E6CEA}{BAC6F6}{DCE6F7}{F1F4F7}
\setaccent{timenet}
\definecolor{tslmInk}{HTML}{0C1836}   % type in the figure navy

\input{math_commands.tex}

\usepackage{amsmath,amssymb}
\usepackage{booktabs}
\usepackage{tabularx}
\usepackage{multirow}
\usepackage{enumitem}
\usepackage{amsmath,amssymb,amsthm,booktabs,makecell,enumitem,tabularx}
 \theoremstyle{definition}\newtheorem{definition}{Definition}
\theoremstyle{plain}\newtheorem{proposition}{Proposition}

\colorlet{famDesc}{tnBlue}    % descriptive inference
\colorlet{famLoc}{tnAmber}    % temporal localization
\colorlet{famFut}{tnGreen}    % future-directed inference
\colorlet{famRel}{tnTeal}     % relational inference
\colorlet{famSyn}{tnPink}     % temporal synthesis
\usepackage{xcolor,tikz}
\usepackage[edges]{forest}
\usetikzlibrary{decorations.pathreplacing,calc}

\forestset{
  slot/.style={draw=tslmFaint, dashed, line width=0.5pt, fill=white, text=tslmMuted,
               font=\scriptsize\itshape, node options={align=center},
               edge={draw=tslmFaint, dashed, line width=0.5pt}},
}

\usepackage{hyperref}
\usepackage[nameinlink,capitalise]{cleveref}
\usepackage{url}
\usepackage{array}
\usepackage{amsfonts}
\usepackage{nicefrac}
\usepackage{microtype}
\usepackage{graphicx}
\usepackage{float}

\usepackage[acronym]{glossaries}
\makeglossaries

\usepackage{listings}
\usepackage{fancyvrb}
\input{listings/pygments-tango}
\usepackage[most]{tcolorbox}

\newtcolorbox{apipanel}[1]{
  enhanced,
  title={#1},
  fonttitle=\footnotesize\sffamily\fontseries{sb}\selectfont,
  coltitle=tslmInk,
  colbacktitle=tslmAccentWash,
  colback=white,
  colframe=tslmRule,
  boxrule=0.5pt,
  arc=3pt,
  left=4pt,
  right=4pt,
  top=3pt,
  bottom=3pt,
  boxsep=0pt,
  toptitle=2pt,
  bottomtitle=2pt,
  equal height group=api,
}

\input{acronyms}

\usepackage{ragged2e}

\title{\titlehl{TimeNet}: An Extensible Unified Data Infrastructure for
       Next-Generation Temporal Foundation Models}
\shorttitle{TimeNet: Unified Data Infrastructure for Temporal Foundation Models}
\papertype{Preprint}
\date{\today}

\author{%
\au{Martin Maritsch}{1,2,*},
\au{Timo Stoffregen}{1,*},
\au{Thomas Kaar}{1,2,3},
\au{Behsad Riemer}{4},
\au{Maxwell A. Xu}{5},
\au{Max Rosenblattl}{2},
\au{Juncheng Liu}{6},
\au{Nicolas Zumarraga}{2},
\au{Yu Yvonne Wu}{7},
\au{Denys Herasymuk}{8},
\au{Sparsh Rastogi}{9},
\au{Hyungjun Yoon}{10},
\au{Bosong Huang}{11},
\au{Arvind Pillai}{7},
\au{Dmytro Lopushanskyy}{12},
\au{Tony Chen}{3,13},
\au{Robin Deuber}{2},
\au{Yichen Liu}{2,4},\\
\au{Shvat Messica}{14},
\au{Dan Li}{15},
\au{Jian Lou}{15},
\au{Yuwei Zhang}{16},
\au{Jaeho Kim}{17},
\au{Ren\'ee Rosillo Garcia}{18},
\au{Fan Wu}{2},\\
\au{Robert M\"uller}{2},
\au{Elgar Fleisch}{2,19},
\au{Flora D. Salim}{9},
\au{Dimitris Spathis}{5,16},
\au{Yuzhe Yang}{20},
\au{Aaqib Saeed}{21},\\
\au{Daniel McDuff}{5},
\au{Ming Jin}{11},
\au{Markus Kreft}{2},
\au{Kevin O'Sullivan}{2},
\au{Robert Jakob}{1,2}
\\
\au{Azul Garza}{18,\textdagger},
\au{Paul Schmiedmayer}{3,\textdagger},
\au{Patrick Langer}{1,2,3,\textdagger}%
}

\affiliations{%
\aff{1}{Aionic Labs, Switzerland}
\aff{2}{Agentic Systems Lab, ETH Zurich, Switzerland}
\aff{3}{Stanford University, USA}
\aff{4}{Technical University of Munich, Germany}
\aff{5}{Google}
\aff{6}{National University of Singapore, Singapore}
\aff{7}{Dartmouth College, USA}
\aff{8}{New York University, USA}
\aff{9}{UNSW Sydney, Australia}
\aff{10}{KAIST, South Korea}
\aff{11}{Griffith University, Australia}
\aff{12}{Seattle Children's Hospital, USA}
\aff{13}{Columbia University, USA}
\aff{14}{Harvard University, USA}
\aff{15}{Sun Yat-sen University, China}
\aff{16}{University of Cambridge, UK}
\aff{17}{University of Virginia, USA}
\aff{18}{TimeCopilot, USA}
\aff{19}{Centre for Digital Health Interventions, St.\ Gallen, Switzerland}
\aff{20}{University of California, Los Angeles, USA}
\aff{21}{Eindhoven University of Technology, The Netherlands}%
}

\authornotes{\textsuperscript{*}Equal contribution.
  \qquad\textsuperscript{\textdagger}Shared senior authors.}

\correspondence{\email{planger@stanford.edu}}
\repository{https://github.com/OpenTSLM/TimeNet}
\huggingface{https://huggingface.co/OpenTSLM}
\projectpage{https://timenet.ai}
\coverlogos{\includegraphics[width=\linewidth]{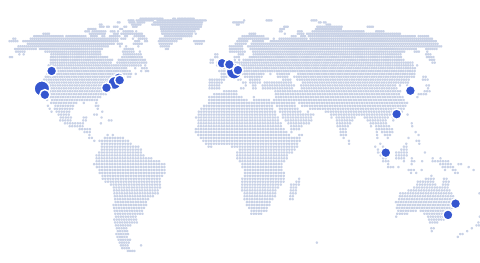}}
\begin{document}

\maketitle

\begin{abstract}
\Glspl{tfm} aim to generalize across domains, datasets, and tasks. Yet, their development remains constrained by fragmented, task-specific data formats, annotations, and processing pipelines.
We introduce \projectname*{TimeNet}, an open-source data standard and scalable infrastructure that decouples temporal data from task definitions and represents signals, metadata, annotations, and supervision in a shared, extensible data model.
\projectname{TimeNet} supports multimodal signals with regular, irregular, or ordinal time axes and expresses different task families (including classification, forecasting, temporal localization, question answering, generation, and editing) as reusable views over the same recordings.
This shared representation enables heterogeneous time-series datasets to be combined for large-scale model training across domains, modalities, and tasks.
We demonstrate \projectname{TimeNet} by transcoding datasets with 1.5M task instances spanning diverse domains, modalities, temporal scales, and forms of supervision, while retaining practical I/O performance relative to native formats.
\projectname{TimeNet} enables an existing \gls{tfm} training pipeline to support joint training on a configurable number of heterogeneous datasets through configuration changes alone. We show this capability by training \glspl{tslm} across multiple datasets and tasks, obtaining a 14\% F1 score improvement compared with models trained on individual datasets.
These results show that \projectname{TimeNet} provides the data and systems foundation needed to move beyond task- and dataset-specific TFMs toward models that can learn jointly across heterogeneous domains, modalities, temporal scales, and forms of supervision from a common data model.
\end{abstract}
\glsresetall

\input{01_introduction}

\input{02_related_work}

\input{03_methods}

\input{04_experiments}
\input{05_discussion}

\subsubsection*{Reproducibility Statement}
All source code associated with this work is publicly available. All external datasets used are open source. These resources ensure full reproducibility of our results.

\subsubsection*{Use of AI}
In this work, we used generative AI tools to implement parts of the software code used in our experimental pipeline and to assist with the formal expression of portions of our task definitions, including mathematical notation and formulation. The underlying concepts, task taxonomy, and scientific content of these definitions were developed by the authors. We have not used generative AI tools to generate synthetic datasets, develop theoretical models or conceptual frameworks, provide critical ingredients for or assist in writing mathematical proofs, propose or refine hypotheses, design or provide feedback on research methodology or experiments, assist with translation, clean or reformat datasets, support qualitative or thematic data analysis, interpret results, or draw scientific conclusions. Additionally, we used generative AI tools to edit portions of the manuscript to improve grammar, clarity, and readability; to review source code for potential errors or bugs; and to create or modify scientific figures or images.

We have reviewed all AI-assisted work. In particular, AI-assisted code was developed under strict human supervision, reviewed by at least two authors through pull requests before being accepted into the repository, and tested for correctness before being used in experiments. AI-assisted text, mathematical formulations, and visual materials were likewise reviewed and approved by the authors.

We take responsibility for the final content of this work, including text, claims, code, figures, and other artifacts produced or modified with the aid of generative AI.

% Variant B: world map with the institution names (see affiliations.tex)
\input{affiliations}

\bibliography{iclr2027_conference}
\bibliographystyle{iclr2027_conference}   % natbib author-year; plainnat also works

\appendix
\input{06_appendix}
\input{07_appendix_task_taxonomy}

\printglossary[type=\acronymtype]

\end{document}

%% file: math_commands.tex
\usepackage{amsmath,amsfonts,bm}

\def\eqref#1{equation~\ref{#1}}
\def\1{\bm{1}}

\DeclareMathAlphabet{\mathsfit}{\encodingdefault}{\sfdefault}{m}{sl}
\SetMathAlphabet{\mathsfit}{bold}{\encodingdefault}{\sfdefault}{bx}{n}

%% file: acronyms.tex
\newacronym{tsfm}{TSFM}{Time Series Foundation Model}
\newacronym{tfm}{TFM}{Temporal Foundation Model}
\newacronym{ecg}{ECG}{Electrocardiogram}
\newacronym{tslm}{TSLM}{Time Series Language Model}
\newacronym{meds}{MEDS}{Medical Event Data Standard}
\newacronym[first={SHHS}]{shhs}{SHHS}{Sleep Heart Health Study}
\newacronym{api}{API}{application programming interface}
\newacronym[first={AI}]{ai}{AI}{artificial intelligence}
\newacronym[first={COVID-19}]{covid}{COVID-19}{coronavirus disease 2019}
\newacronym{cot}{CoT}{chain-of-thought}
\newacronym{crps}{CRPS}{continuous ranked probability score}
\newacronym{csv}{CSV}{comma-separated values}
\newacronym{dtw}{DTW}{dynamic time warping}
\newacronym{edf}{EDF}{European Data Format}
\newacronym{eeg}{EEG}{electroencephalogram}
\newacronym{emg}{EMG}{electromyogram}
\newacronym{eog}{EOG}{electrooculogram}
\newacronym{fscore}{F1}{F1 score}
\newacronym{http}{HTTP}{Hypertext Transfer Protocol}
\newacronym[first={I/O}]{io}{I/O}{input/output}
\newacronym{iou}{IoU}{intersection over union}
\newacronym[first={JSON}]{json}{JSON}{JavaScript Object Notation}
\newacronym{mae}{MAE}{mean absolute error}
\newacronym{mape}{MAPE}{mean absolute percentage error}
\newacronym{mase}{MASE}{mean absolute scaled error}
\newacronym{mrr}{MRR}{mean reciprocal rank}
\newacronym{nll}{NLL}{negative log-likelihood}
\newacronym{qa}{QA}{question answering}
\newacronym[first={SDK}]{sdk}{SDK}{software development kit}
\newacronym[first={SQL}]{sql}{SQL}{Structured Query Language}
\newacronym[first={UTF}]{utf}{UTF}{Unicode Transformation Format}
\newacronym{wfdb}{WFDB}{WaveForm DataBase}
\newacronym[first={MMD}]{mmd}{MMD}{multi-domain multimodal dataset}
\newacronym[first={TSB-UAD}]{tsbuad}{TSB-UAD}{time-series benchmark for univariate anomaly detection}
\newacronym[first={UCR}]{ucr}{UCR}{University of California, Riverside}
\newacronym[first={UEA}]{uea}{UEA}{University of East Anglia}
\newacronym[first={UTSD}]{utsd}{UTSD}{Unified Time Series Dataset}
\newacronym[first={ECG-QA}]{ecgqa}{ECG-QA}{Electrocardiogram Question Answering}

%% file: 01_introduction.tex
\section{Introduction}\label{intro}

% ====================== V2, shortened ======================
Recent advances in \glspl{tfm}, including \glspl{tsfm} and \glspl{tslm}, have demonstrated that scaling data, models, and compute can enable efficient transfer and zero/few-shot generalization across heterogeneous time-series datasets, tasks, and application domains~\citep{Narayanswamy26arxiv, yao2025iclr, Ansari25arxiv,Goswami24icml, shi2024neurips}. 
However, most existing \glspl{tsfm} remain specialized to individual tasks, such as forecasting, classification, or anomaly detection, or to particular domains (e.g., weather or health).
Real-world temporal systems are often multimodal, asynchronous, event-driven, multi-scale, and subject to distribution shifts ~\citep{Aksu24arxiv,Garza26arxiv}.
%understanding them may require multiple complementary capabilities as no single task captures all aspects of their state, dynamics, and context. 
Modeling these systems may require the ability to infer the current state of a system, explain salient events \citep{Rastogi2026arxiv}, anticipate future trajectories, and reason about how those trajectories may change under different interventions. For example, monitoring a patient in an intensive care unit can involve detecting the onset of an arrhythmia from \gls{ecg}, explaining the waveform changes that support it, forecasting the patient's hemodynamic trajectory, and reasoning about how that trajectory would change under an intervention.
Next-generation \glspl{tfm} should therefore move beyond task-specific variants toward shared architectures that can model evolving systems across tasks, modalities, and temporal scales.

\begin{figure}[h!t]
  \centering
  \includegraphics[width=1.0\linewidth]{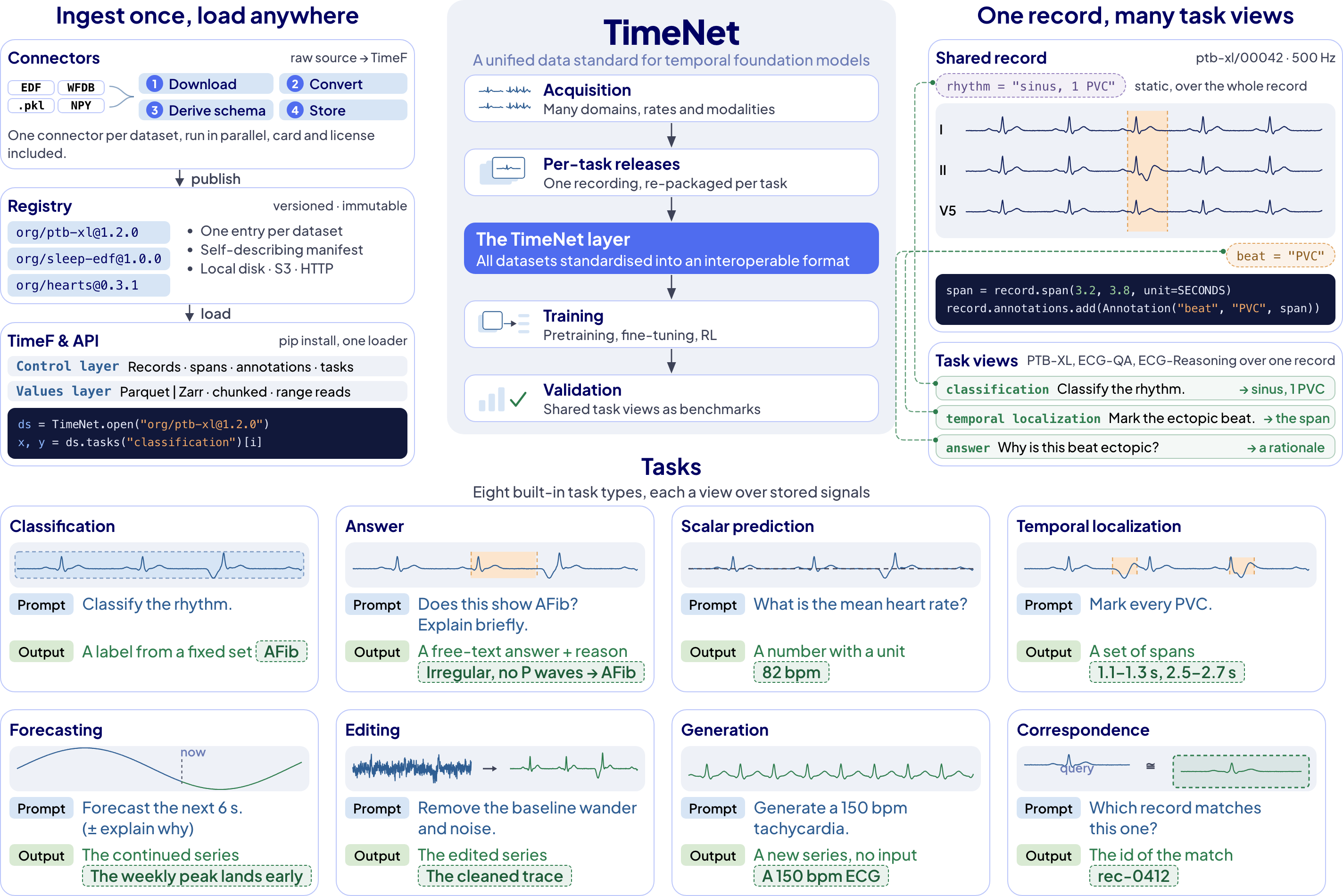}
 \caption{\textbf{An Overview of TimeNet.} We introduce a unified data standard, open software infrastructure and an extensive framework for training and evaluation. Multiple tasks can be defined over the same underlying data and across datasets, which can be combined seamlessly during training.}
  \label{fig:timenet_annotations_tasks}
\end{figure}

Building such models presents two core challenges: (1) learning from heterogeneous forms of supervision across pretraining, supervised fine-tuning, instruction tuning, and post-training with verifiable rewards; and (2) integrating datasets that vary in domain, acquisition protocol, temporal scale, sensor configuration, modality, and tasks. Existing datasets fail to support this vision, since they are typically published as self-contained, task-specific packages with separate and ad-hoc preprocessing pipelines, representations, annotations, and loaders.
When new tasks are introduced, the same underlying recordings are often repackaged into new task-specific datasets with custom formats, rather than extending a shared data model.
PTB-XL~\citep{Wagner20scidata}, for example, was originally released for \gls{ecg} classification and later repurposed for question answering in \gls{ecg}-\gls{qa}~\citep{Oh23nipsdb}, clinical reasoning in \gls{ecg}-Reasoning-Benchmark~\citep{Oh26arxiv}, and instruction tuning in PULSE~\citep{Liu24arxiv}, each effort introducing its own task formulation and representation of the same underlying \gls{ecg} recordings.
These efforts add supervision to existing recordings, but remain specialized rather than building on a shared representation.

%Language and vision faced analogous challenges.
%WordNet provided natural language research with a common lexical ontology~\citep{Miller95cacm}, while ImageNet built on WordNet’s hierarchy to organize a large-scale visual corpus, enabling the computer vision boom of the 2010s.
To address these fragmentation challenges and provide a shared foundation for the research field, we introduce \textbf{TimeNet}, a unified and extensible data standard and infrastructure for next-generation \glspl{tfm}. TimeNet makes three core contributions (Figure~\ref{fig:timenet_annotations_tasks}): \textbf{(1) A Unified Data Standard:} a common, model-independent format that separates temporal recordings from task definitions, enabling classification, forecasting, anomaly detection, reasoning, generation, and editing as tasks over a shared data representation; \textbf{(2) Open Software Infrastructure:} tools for submitting, cataloging, processing, versioning, splitting, and loading datasets through a unified interface; and \textbf{(3) An Extensible Foundation for Scalable Training and Benchmarking Across Datasets:} a plug-and-play framework that converts any supported dataset into training tensors without dataset-specific code, enabling to integrate and scale diverse datasets and tasks from a unified representation for large-scale training of multitask \glspl{tfm} and standardized evaluation. We demonstrate this capability through pooled-dataset training of \glspl{tslm}, providing a concrete use case of TimeNet’s model-independent infrastructure for \glspl{tfm}.
We open-source all implementations and experiments\footnote{Anonymized during review; code attached as supplementary material.}.

%% file: 02_related_work.tex
\section{Related Work}
%In the interest of brevity, we have banished the full literature review to the Supplement \ref{}.
%The key distinction for positioning TimeNet, however, is the level at which prior work standardizes time-series data.
Time Series resources range from task-specific benchmarks, such as \acrshort{ucr}/\acrshort{uea} for classification, Monash for forecasting, and \acrshort{tsbuad} for anomaly detection~\citep{Dau19jas,Bagnall18arxiv,Godahewa21nipsdb,Paparrizos22pvldb}, to large heterogeneous corpora including Time Series Pile, LOTSA, and Time-\acrshort{mmd}~\citep{Goswami24icml,Woo24icml,Liu24nipsdb}. Forecasting evaluation has evolved from the M-competitions~\citep{Makridakis20ijf} toward broader suites such as GIFT-Eval, \gls{tsfm}-Bench, TIME, BOOM, \texttt{fev}, and live evaluation streams such as Impermanent~\citep{Aksu24arxiv,Li25kdd,Qiao26arxiv,Cohen25arxiv,Shchur25arxiv,Garza26arxiv}. Recent benchmarks also target reasoning, retrieval, generation, and editing~\citep{Cai24arxiv,Li26arxiv,Zumarraga26iclrw,Zumarraga26arxiv,Langer25arxiv,Shuai26arxiv,Qiu25arxiv,Kim26arxiv, Kong2025acl}. Complementary efforts standardize software interfaces, dataset formats, or observational data~\citep{Loning19arxiv,Alexandrov20jmlr,Herzen22jmlr, prabowo2024neuripsbts, Wu26arxiv,Lansing22oceans,Arnrich24iclrw}, while corpora such as \acrshort{utsd} and LOTSA harmonize numerical data for pretraining~\citep{Liu24icml,Woo24icml}. Agentic frameworks such as TimeCopilot~\citep{Garza25arxiv} provide a unified \acrshort{api} across \glspl{tsfm}, but still rely on upstream datasets in compatible formats.

While these efforts focus on curating and evaluating datasets, few address the infrastructure for storing and serving time series at scale for model training, and many existing solutions target only longitudinal (physiological) waveforms. AtriumDB~\citep{goodwin2020practical} collects waveforms from diverse devices and vendors into a unified collection, stores them losslessly using a compact delta-encoding compression scheme, indexes them with metadata for quality control and cohort identification without accessing raw data, and enables rapid programmatic retrieval optimized for parallel and distributed computing. Its open-source \acrshort{sdk} provides efficient data loaders with batching strategies tailored to physiological waveforms, and the platform has enabled use cases such as real-time bedside models on high-frequency blood-pressure signals~\citep{nagaraj2024needles}. However, such infrastructure is limited to a single data type and clinical environment, lacks tasks and first-class text supervision, and cannot publish heterogeneous datasets as versioned, shareable entries. %TimeNet addresses these gaps by spanning arbitrary domains and modalities, treating annotations and tasks as reusable views over shared recordings, and publishing datasets to a public versioned registry.

%% file: 03_methods.tex
\section{Methods}
\label{sec:methods}
TimeNet addresses limited interoperability, extensibility, and scalability by providing a unified schema and infrastructure that can represent heterogeneous time-series datasets without losing information. \input{figures/timenet_pipeline}Its design is guided by four key principles:
\textbf{(1) Unified interface:} datasets are accessed through a common schema and \acrshort{api}, enabling a single processing and loading pipeline.
\textbf{(2) Extensibility:} records can be augmented with new tasks and annotations without duplicating source data; new task types can be added as needed.
\textbf{(3) Many-to-many relationships:} tasks and annotations can reference multiple records, and vice versa.
\textbf{(4) Composability:} records, tasks, and annotations can be independently combined and reused across datasets.

We translate these principles into the various components that form a pipeline, as shown in Figure~\ref{fig:timenet_pipeline}.
First, dataset-specific connectors standardize heterogeneous source datasets into the \texttt{TimeF} data model, which provides a common representation for signals, annotations, tasks, and their relationships, and stores converted datasets as reusable, versioned entries (Section~\ref{sec:methods:data_representation}).
Second, we define a common framework for temporal tasks that makes the structure and relationships between task types explicit. This framework guides TimeNet's built-in task definitions and provides an \acrshort{api} through which existing tasks can be specialized or new task types can be added without changing the underlying data representation (Section~\ref{sec:taskspec}).
Finally, the data-loading layer operates over the resulting common representation and can select, combine, and collate data from arbitrary datasets into model-specific inputs, independent of their original source format (Section~\ref{sec:methods:data_collation}).

Because dataset ingestion, task definition, and data loading all operate within the same infrastructure and based on a common data model, improvements made at any layer compound across the ecosystem rather than remaining dataset-specific. More generally, TimeNet changes dataset integration with consumers (for example, model architectures or data analysis libraries) from a pairwise problem into an interface problem. With \(N\) datasets and \(M\) downstream consumers, directly supporting every dataset--consumer combination may require up to $N\times M$ dataset-specific integrations. TimeNet instead requires $N$ dataset connectors and $M$ integrations against the shared TimeF interface, reducing the integration surface to $M+N$. New loading logic, tasks, annotations, batch collation or efficient data distribution can be reused across all compatible datasets, while newly added datasets immediately benefit from the existing infrastructure. As TimeNet grows, both data and functionality therefore accumulate around the same interface, reducing per-dataset integration effort and enabling increasingly large and heterogeneous \gls{tfm} training.

\subsection{Data Representation}
\label{sec:methods:data_representation}
We introduce \texttt{TimeF} as underlying data model for TimeNet. \texttt{TimeF} provides an explicit object hierarchy embedded in a relational representation spanning signals, annotations, and tasks and maps them to an efficient on-disk format.
Figure~\ref{fig:data_control_hierarchy} illustrates how the hierarchy maps to the control and value planes of the on-disk format.
\input{figures/data_control_hierarchy}

\textbf{Records, Sources, and Signals:}
In TimeNet, a TimeF-Dataset contains Records, each of which represents a logical observation such as a patient visit, an experimental session, or a machine run.
Rather than containing Signals directly, each Record owns one or more Sources that describe where its data originated.
Sources form a recursive hierarchy: a Source may contain child Sources as well as Signals, thereby preserving acquisition structures such as devices, or sensor assemblies without encoding provenance in Signal names.
Each Signal specifies its data type, shape, unit, and an \emph{axis} along which it is sampled.
Axes can be regular (uniform fixed frequency), irregular (non-uniformly sampled), or ordinal (ordered values without timestamps).
Signal specifications and axes are immutable and may be shared, while each Signal remains a leaf of a single Record's Source hierarchy.
The resulting hierarchy provides a flexible representation of heterogeneous recordings in which the atomic components of a dataset can be composed in arbitrary configurations.
  
\textbf{Spans:}
A Span identifies a point or half-open interval, either in microseconds on a Record timeline or by ordinal position within a Signal (see Figure~\ref{fig:data_control_hierarchy} in green). Time-valued Spans may additionally identify the Signals to which they apply and are used to localize annotations, Task inputs, and temporal targets. This common addressing scheme makes temporal regions reusable across Annotations and Tasks, independently of how the underlying Signals are sampled.

\textbf{Annotations:}
Annotations represent typed observations about Datasets, Tasks, Records, Sources, or Signals and may optionally be localized by a Span.
TimeF separates reusable annotation content from its occurrences, allowing the same content to be attached to several objects while retaining occurrence-specific provenance, confidence, and metadata.
This design allows supervision and metadata to be stored in shared data objects rather than being copied into task-specific dataset variants.

\textbf{Tasks:}
TimeNet defines eight task types: \emph{classification}, \emph{question answering}, \emph{scalar prediction}, \emph{temporal localization}, \emph{forecasting}, \emph{generation}, \emph{editing}, and \emph{correspondence}.
Tasks reference ordered tuples of input Records and may encode their targets either inline or through references to annotation occurrences, thereby avoiding duplication of the underlying data and supervision.
Consequently, the same recordings can support multiple forms of supervision and can be recombined across training and evaluation settings without repackaging the underlying Signals (see Figure~\ref{fig:timenet_pipeline}).
Further details and API examples are provided in Appendix~\ref{sec:app:implementation}.

\textbf{Object and storage models:}
TimeNet separates its user-facing object model from the concrete storage model used by TimeF.
The object model described above provides an abstract representation of a dataset through Records, Sources, Signals, Annotations, Tasks, and the relationships among them, and can be used through an \gls{api}.
TimeF maps these objects and relationships to an immutable storage representation comprising a \acrshort{json} \emph{manifest}, a \emph{control plane} database, and the \emph{values plane}; a collection of Signal value files.
Within this representation, the control plane is managed with DuckDB, a portable analytical \acrshort{sql} database engine. The database contains relational tables storing the object hierarchy, axes, annotations, Tasks, and their references.
The larger Signal arrays containing actual data form the values plane and are written as \textit{Parquet} format.
A chunk index in the control plane links each Signal to its physical values, thereby connecting the abstract object model to the underlying data without embedding those arrays in the relational database.
This mapping permits relational filtering over metadata and object types without scanning large arrays, while retaining efficient parallel and partial access to Signal values.

%In combination, these abstractions fulfil the four design requirements introduced above.
%The object hierarchy model provides a unified interface, reusable Annotations and Tasks support extensibility and composability, and the DuckDB control plane makes the relationships among these objects explicit and queryable.

\textbf{Dataset Ingestion and Registry}
\label{sec:registry}
We provide dataset-specific connectors that transform source datasets into TimeNet via a \texttt{download}$\to$\texttt{convert}$\to$\texttt{derive\_schema}$\to$\texttt{store} pipeline, enabling parallel record processing before dataset-level aggregation.
Converted datasets can be published as immutable, versioned entries in the \textbf{TimeNet registry}, identified by \texttt{org/name[@version]}, subject to dataset licensing.
Restricted, non-redistributable datasets can instead be rebuilt locally using a publishable connector.
The registry currently contains 1.51M task instances: 899k question answering, 484k classification, 131k generation, 441 temporal localization, and 292 scalar prediction.

\input{03_2_specification_language}

\subsection{Data collation}
\label{sec:methods:data_collation}

Once datasets are represented in TimeF, the same loading and collation pipeline can be applied independent of their original source format.
For a batch of Tasks $\mathcal{B}=\{\tau_i\}_{i=1}^{B}$, TimeNet resolves the Records, Signals, Annotations, temporal scopes, and targets referenced by each $\tau_i$, and materializes only the required Signal values.
A model-specific collation function
%\begin{equation}
 $   \mathcal{C}_m(\mathcal{B})
    =
    \bigl(X_m, M_m, C_m, Y_m\bigr)$
%\end{equation}
then maps the resolved task instances into the representation expected by model $m$, where $X_m$ contains the time-series inputs, $M_m$ optionally identifies valid or observed values, $C_m$ contains additional conditioning such as text or metadata, and $Y_m$ contains the training targets. Figure~\ref{fig:prompt_rendering} visualizes this process.\input{figures/prompt_rendering}Depending on the model, this step may include windowing, padding, resampling, channel selection, or conversion into one or more tensors.
Because $\mathcal{C}_m$ operates on the common TimeF objects rather than dataset-specific formats, the same batching and tensorization logic can be reused across datasets and task types, and newly added datasets can directly use the existing training infrastructure, enabling to scale up datasets without having to write extra code.

%% file: figures/timenet_pipeline.tex
\begin{figure}[!b]
  \centering
 % \includesvg[
  %  width=\linewidth,
  %  inkscapelatex=false
  %]{figures/timenet_pipeline_figma.p}
 \includegraphics[width=\linewidth]{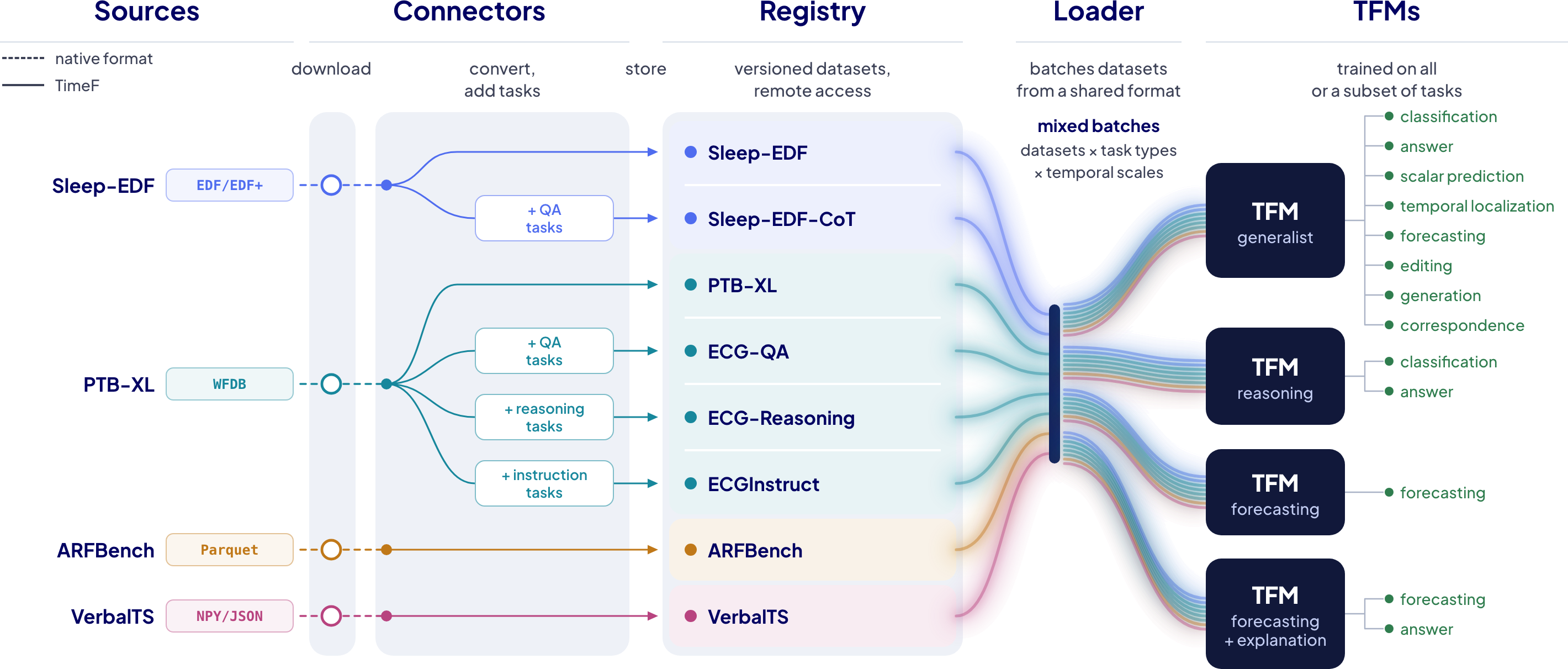}
 \caption{\textbf{TimeNet processing pipeline.}
 TimeNet allows for flexible reuse and recombination of heterogeneous data.
 Once ingested into TimeNet using a specialized connector, a dataset can be used for any model that operates over applicable Tasks.
 }
  \label{fig:timenet_pipeline}
\end{figure}

%% file: figures/data_control_hierarchy.tex
\begin{figure}[h!t]
  \centering
  % Pre-rendered PDF (fonts embedded) -- no Inkscape or installed fonts needed.
  % To render from the SVG instead (Inkscape + Plus Jakarta Sans required), use:
  % \includesvg[width=\linewidth,inkscapelatex=false]{figures/data_control_hierarchy_aligned.svg}
  \includegraphics[width=\linewidth]{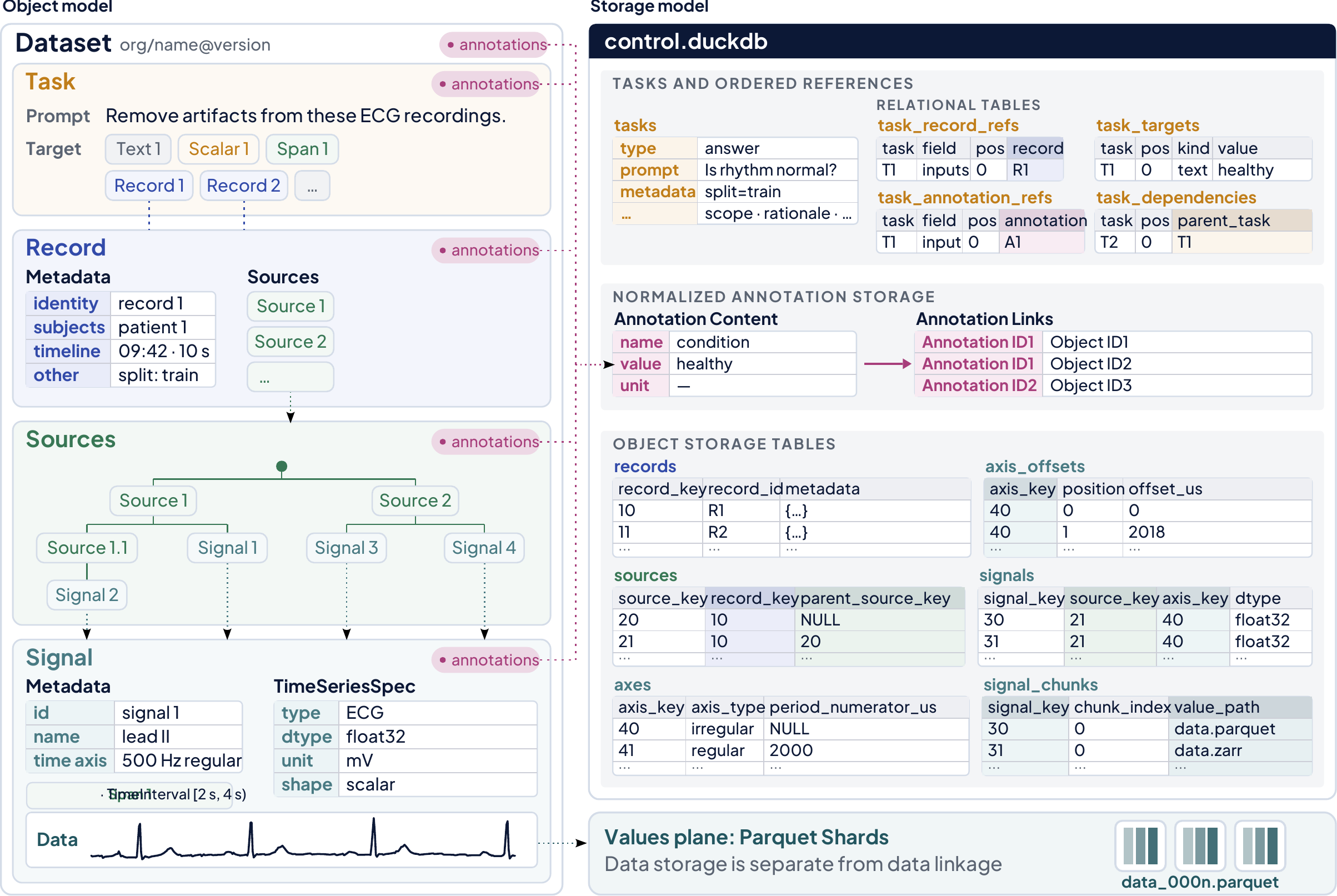}
  \caption{\textbf{TimeNet data model.}
    TimeNet represents each dataset as Records with recursive hierarchies of Sources and Signals.
    A Source describes where data originated and can contain child Sources.
    Signals describe value streams and link them to stored values and metadata.
    Tasks reference one or more Records, and each Record can participate in multiple Tasks.
    Datasets, Tasks, Records, Sources, and Signals can all have annotations.
    \textit{DuckDB} stores the hierarchy, annotations, and many-to-many relationships.
    %This relational structure supports flexible filtering and avoids data duplication through reference-based links.
    Underlying Signal values are stored separately in efficient, sharded binary files.% - we prefer \textit{Parquet} for its interoperability, and efficient partial, parallel, and remote access - but the values-plane abstraction remains accommodating to alternative storage formats.
  }
  \label{fig:data_control_hierarchy}
\end{figure}

%% file: 03_2_specification_language.tex
\subsection{Formal Specification of Temporal Tasks}
\label{sec:taskspec}

Time-series tasks are commonly organized by names such as \textit{classification},
\textit{forecasting}, \textit{anomaly detection}, or \textit{question answering}. However, these names
often mix several distinct aspects of a prediction task: \emph{what} the
model must predict, \emph{where} that target lies relative to the observed
signal, and \emph{how} the task is presented to the model. For example, an
anomaly may be represented either as a record-level label or as a temporal
interval, while the same prediction may be requested through a structured
output or a natural-language question. As new formulations emerge, the same
underlying recordings are therefore frequently repackaged into separate
task-specific datasets (Section~\ref{intro}). In TimeNet, we instead derive tasks
compositionally from their inputs, targets, and temporal relationships.

\paragraph{Task representation.}
A record $r$ contains Signals $x_k:\mathcal{T}_k\to\mathbb{R}^{c_k}$,
$k\in K_r$, together with annotations $\mathcal{A}_r$. A task selects an
input view $W_{\mathrm{in}}$ over one or more records, optional conditioning
$C$, and a target of type $O$ associated with support $W_{\mathrm{out}}$.
The target may be stored directly or derived from existing signals and
annotations. The complete TimeNet task specification, including interface,
rationale, access protocol, and scoring, is given in
Appendix~\ref{app:taxonomy}.

To characterize the semantic operation performed by a task, we retain three
properties of this specification. Let
$\kappa(O)\in\{\textsc{Value},\textsc{Span},\textsc{Signal},\textsc{Index}\}$
denote the kind of target. Let
\[
\Delta(\tau)\in
\{\textsc{Obs},\textsc{Gap},\textsc{Fut},\textsc{Past},
  \textsc{Chan},\textsc{Null}\}
\]
describe where the target lies relative to the observed input: on observed
support, within a missing region, beyond the observation in the future or
past, in an unobserved channel, or with no observed input, respectively. Finally,
$\alpha(\tau)\in\{1,\geq2\}$ denotes the number of temporal objects over
which the target is jointly defined. We call

\begin{equation}
    \phi(\tau)
    =
    \bigl(\kappa(O),\Delta(\tau),\alpha(\tau)\bigr)
    \label{eq:task-signature}
\end{equation}
the \emph{task signature}.

Different temporal tasks can be constructed by combining the three components
of the signature: output kind $\kappa$, target position $\Delta$, and arity
$\alpha$. For example, a single-object \textsc{Value} target within the
recorded support gives \texttt{attribute inference}, whereas a \textsc{Span}
target gives \texttt{temporal localization}. Moving the target beyond the
observed context gives \texttt{future-directed inference}, while defining it
jointly over multiple temporal objects gives \texttt{relational inference}.
Signal-valued targets correspond to \texttt{temporal construction} or
\texttt{temporal transformation} depending on whether an observed temporal
input is available. These combinations give the six constructions in
Table~\ref{tab:taxonomy}, grouped into five semantic task families. The complete signature
grid and formal definitions are given in
Appendix~\ref{app:grid}--\ref{app:families}.
\begin{table}[h!t]
\centering\scriptsize
\caption{Compositional taxonomy of temporal tasks. Each family is a region of the signature
space $\phi=(\kappa,\Delta,\alpha)$ (Definitions~\ref{def:task}--\ref{def:support}); the six regions partition
all well-posed signatures (Appendix~\ref{app:grid}). $\ast$: any value; $\Delta_{\mathrm{in}}=\{\textsc{Obs},\textsc{Gap},\textsc{Chan}\}$,
i.e., the target lies within the recorded extent, whether observed, masked, or in an unmeasured channel.
TimeNet realizations in \texttt{monospace}; $+\textsc{Shift}$ etc.\ denote constructors (Definition~\ref{def:constructors}).
For future-directed inference, ``other types'' are \texttt{classification}, \texttt{scalar prediction}, and
\texttt{temporal localization}.}
\label{tab:taxonomy}
\setlength{\tabcolsep}{4pt}
\renewcommand{\arraystretch}{1.25}
\begin{tabularx}{\linewidth}{@{}>{\raggedright\arraybackslash}p{2.35cm} l
  >{\raggedright\arraybackslash}X >{\raggedright\arraybackslash}X >{\raggedright\arraybackslash}p{2.2cm}@{}}
\toprule
Task family \& construction & Signature $\phi$ & Defining relation & Derived tasks & TimeNet realization \\
\midrule
\textbf{Attribute inference}\newline $f(X,C)\to y$ & $(\textsc{Value},\Delta_{\mathrm{in}},1)$
  & Infer a property of the recorded object or a span of it.
  & Classification, regression, state and severity estimation, captioning
  & \texttt{classification}, \texttt{scalar prediction}, \texttt{answer} \\ 
\textbf{Temporal localization}\newline $f(X,q)\to\{s_i\}$, $s_i\subseteq\mathcal{T}$ & $(\textsc{Span},\Delta_{\mathrm{in}},1)$
  & Identify where or when a property or event occurs in the recorded extent.
  & Event and anomaly localization, change-point detection, segmentation
  & \texttt{temporal localization} \\ 
\textbf{Future-directed inference}\newline $f(X_{\le t},C)\to Y_{>t}$ & $(\ast,\{\textsc{Fut},\textsc{Past}\},1)$
  & Infer a target whose support lies beyond the observed context.
  & Forecasting, event and future-state prediction, time-to-event, onset prediction
  & \texttt{forecasting}; other types $+\textsc{Shift}$ \\
\textbf{Relational inference}\newline $f(X_i,X_j,C)\to r$ & $(\ast,\ast,{\geq}2)$
  & Infer a relation between two or more temporal objects or spans.
  & Ordering, comparison, alignment, similarity, retrieval
  & \texttt{correspondence}; any type $+\textsc{Pair}$ \\ 
\textbf{Temporal construction}\newline $f(C)\to X'$ & $(\textsc{Signal},\textsc{Null},1)$
  & Construct a new temporal object from conditioning alone.
  & Conditional and unconditional generation
  & \texttt{generation} \\ \addlinespace[3pt]
\textbf{Temporal transformation}\newline $f(X,C)\to X'$ & $(\textsc{Signal},\Delta_{\mathrm{in}},1)$
  & Produce a temporal object by transforming or completing a recorded one.
  & Editing, denoising, reconstruction, imputation, virtual sensing
  & \texttt{editing}; \texttt{editing} $+\textsc{Mask}$ \\
\bottomrule
\end{tabularx}
\end{table}
This factorization resolves ambiguity in task names. For instance, a record-level anomaly
decision, $X\to\{0,1\}$, has the signature
$(\textsc{Value},\textsc{Obs},1)$ and is therefore attribute inference, whereas
identifying anomalous intervals, $X\to\{[t_s,t_e]\}$, has the signature
$(\textsc{Span},\textsc{Obs},1)$ and is temporal localization. Likewise,
regression and forecasting may both produce numerical targets, but only the
latter places the target beyond the observed context.

\paragraph{Task specialization.}
A task signature specifies the broad semantic family, while more specific
task types can be created by imposing additional constraints. 
For example, \gls{ecg} rhythm classification specializes classification by restricting the input
modality and label space. We denote the set of such task-specific restrictions
by $P(\tau)$, which may constrain the admissible inputs, conditioning
information, outputs, or temporal support. We write

\begin{equation}
    \tau_a \preceq \tau_b
    \quad\Longleftrightarrow\quad
    \phi(\tau_a)=\phi(\tau_b)
    \ \land\
    P(\tau_b)\subseteq P(\tau_a),
    \label{eq:task-specialization}
\end{equation}

so that \gls{ecg} rhythm classification $\preceq$ classification and anomaly
localization $\preceq$ temporal localization. TimeNet records such
relationships between tasks, allowing existing task types to be specialized
without changing the underlying data schema. An
overview of how task families map to built-in TimeNet task types and their
specializations is shown in Figure~\ref{fig:task-tree}.  The same
task may be exposed through a structured interface, natural-language
question, or instruction; may optionally require a rationale; and may be
evaluated offline or in streaming mode. We formalize these axes, together with the task
constructors and their validity conditions, in
Appendix~\ref{app:orthogonal}--\ref{app:preconditions}.

\begin{figure}[t]
    \centering
    \includegraphics[width=\linewidth]{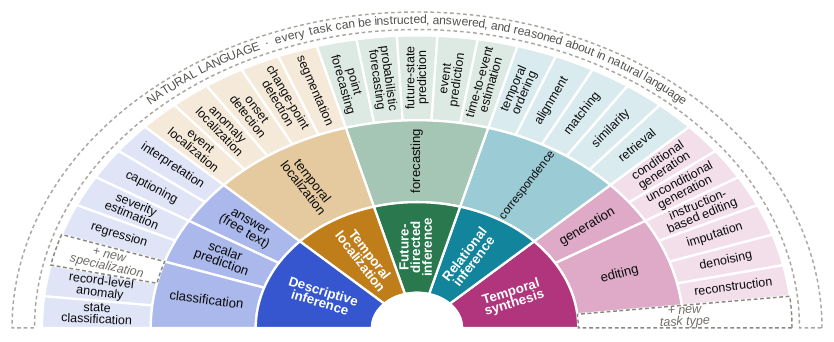}  % tools/make_sunburst.py --half --squash 0.78
    \caption{%
    \textbf{TimeNet task taxonomy.} TimeNet has five semantic task
    families, each containing built-in task types that can be further
    specialized or expanded. Tasks may be expressed in natural language.}
    \label{fig:task-tree}
\end{figure}

%% file: figures/prompt_rendering.tex
\begin{figure}[h!t]
  \centering
  % Pre-rendered PDF (fonts embedded) -- no Inkscape or installed fonts needed.
  % To render from the SVG instead (Inkscape + Plus Jakarta Sans required), use:
  % \includesvg[width=\linewidth,inkscapelatex=false]{figures/prompt_rendering_edges_large_aligned.svg}
  \includegraphics[width=\linewidth]{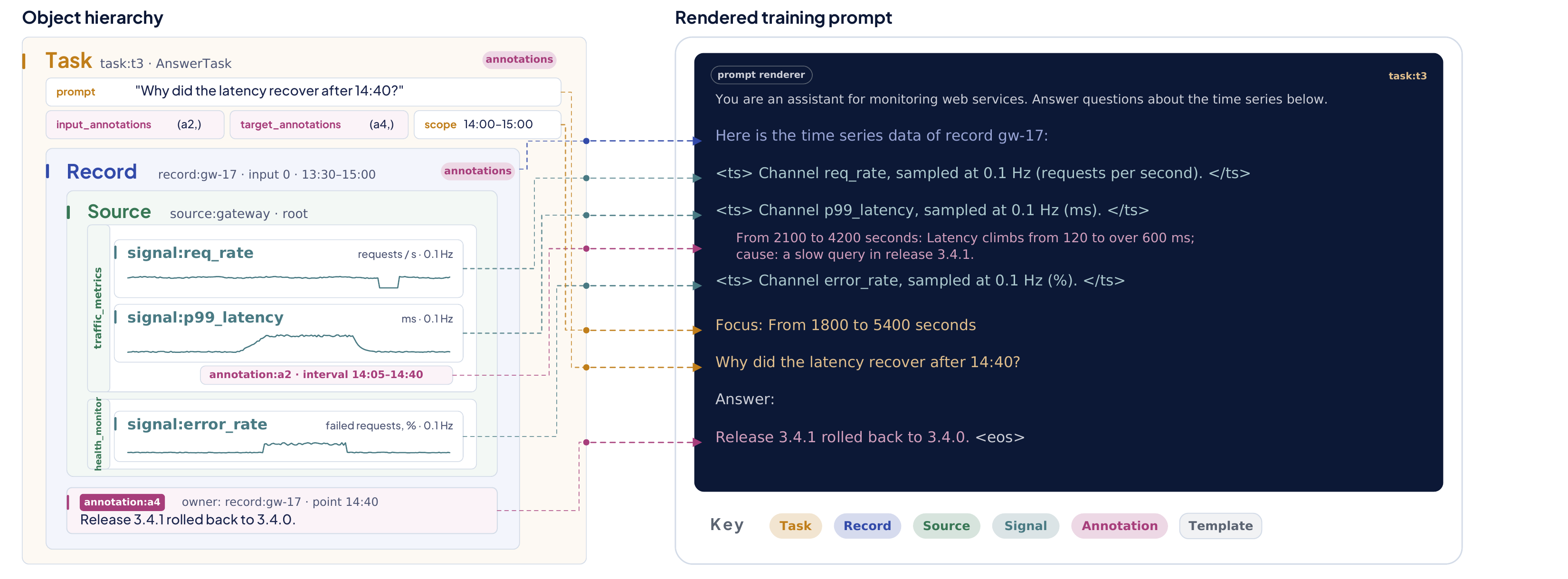}
  \caption{\textbf{Data collation in TimeNet.}
    The structure of TimeNet allows for a common mapping from Tasks to prompts that can be used during training.
  }
  \label{fig:prompt_rendering}
\end{figure}

%% file: 04_experiments.tex
\section{Experiments}
\label{sec:experiments}
To evaluate TimeNet, we investigate whether training jointly on multiple datasets improves performance overall~(Section~\ref{sec:pooled_training}). We also investigate the input/output speed of accessing heterogeneous datasets through a common representation~(Section~\ref{sec:io_performance}). %We first benchmark the I/O throughput of TimeNet against native dataset formats (Section~\ref{sec:io_performance}). We then train OpenTSLM, an open-source time-series language model, jointly across six datasets and evaluate its performance against single-dataset baselines (Section~\ref{sec:pooled_training}).
\subsection{Multi-dataset Training}
\label{sec:pooled_training}
Any TimeNet dataset can be converted into training tensors by the same component, enabling pooled training across datasets. We evaluate whether greater data diversity improves performance using Sleep-\acrshort{edf} Expanded \acrshort{cot}~\citep{Kemp00tbme,Langer25arxiv}, \acrshort{ecgqa} \acrshort{cot}~\citep{Oh23nipsdb,Langer25arxiv}, HEARTS~\citep{Li26arxiv}, and ARFBench~\citep{Xie26arxiv}, which all support question answering. We conduct this comparison for SLIP~\cite{Chen26arxiv} and OpenTSLM~\cite{Langer25arxiv}, training each model family either separately on each dataset (\emph{SLIP Single}, \emph{OpenTSLM Single}) or jointly on all four (\emph{SLIP Pooled}, \emph{OpenTSLM Pooled}). We use the Llama 3.2-1B Flamingo variant for OpenTSLM and the 400M-parameter model for SLIP. Within each family, single-dataset and pooled models use the same architecture and training configuration. All models are evaluated on the same dataset-specific test sets using \gls{fscore} and accuracy.
\input{tab_opentslm_performance}
For OpenTSLM, pooled training increases \gls{fscore} on all four datasets: from 37.49 to 41.80 on Sleep-\acrshort{edf}, 31.19 to 34.21 on \acrshort{ecgqa}, 15.03 to 19.70 on HEARTS, and 32.94 to 37.41 on ARFBench, raising the average \gls{fscore} from 29.16 to 33.28. Accuracy also increases on Sleep-\acrshort{edf} (77.34 to 80.47), \acrshort{ecgqa} (67.19 to 73.44), and HEARTS (35.71 to 38.72), while decreasing slightly on ARFBench (33.33 to 32.82), resulting in an average accuracy increase from 53.39 to 56.36. For SLIP, performance rises on Sleep-\acrshort{edf} and HEARTS, and accuracy increases on ARFBench, but drops slightly on \acrshort{ecgqa}, from 33.93 to 30.00 \gls{fscore} and from 40.01 to 31.32 accuracy. 

\subsection{I/O Performance Across Heterogeneous Datasets}
\label{sec:io_performance}
We benchmark TimeNet on six heterogeneous datasets: Sleep-\acrshort{edf} Expanded \acrshort{cot}~\citep{Kemp00tbme, Langer25arxiv},
\acrshort{ecgqa} \acrshort{cot}~\citep{Oh23nipsdb,Langer25arxiv}, 
PTB-XL~\citep{Wagner20scidata}, HEARTS~\citep{Li26arxiv},
ARFBench~\citep{Xie26arxiv}, VerbalTS~\citep{Gu25icml}, and
SLIP~\citep{Chen26arxiv}. 
We compare the original data format with the TimeNet version across pandas and PyTorch consumers (see Appendix~\ref{sec:appendix:datasets}). Where available, we use original loaders; otherwise, we implement equivalent loaders and verify parity before benchmarking. We provide results in Table~\ref{tab:results:io_performance}. TimeNet is not strictly faster than native formats, which often rely on highly-optimized, data-specific, and memory-mapped encodings. However, TimeNet performance remains in comparable ranges, while providing greater flexibility and reusability.
% TimeF reduces full-read time by approximately $6.2$--$18.8\times$ on HEARTS and ARFBench, while the original representations are $1.5$--$10.8\times$ faster on VerbalTS, SLIP, and Sleep-\acrshort{edf}. TimeF reaches the first item sooner for VerbalTS and SLIP, but later for HEARTS, ARFBench, and Sleep-\acrshort{edf}.
% The same pattern appears in transferred data and HTTP request counts. 
% TimeF reduces both for HEARTS and ARFBench, while increasing them for VerbalTS, SLIP, and Sleep-\acrshort{edf}. 
% These differences depend on how each original dataset is stored and how its canonical items map to the TimeF layout. 

\input{tab_read_performance}

%% file: tab_opentslm_performance.tex
\begin{table*}[h!t]
\centering
\footnotesize
\caption{
OpenTSLM and SLIP performance when trained on each dataset individually
versus jointly.
}
\label{tab:data_pooling}

\resizebox{\textwidth}{!}{
\begin{tabular}{lcccccccccc}
\toprule
&
\multicolumn{2}{c}{\textbf{Sleep-EDF}} &
\multicolumn{2}{c}{\textbf{ECG-QA}} &
\multicolumn{2}{c}{\textbf{HEARTS}} &
\multicolumn{2}{c}{\textbf{ARFBench}} &
\multicolumn{2}{c}{\textbf{Average}} \\
\cmidrule(lr){2-3}
\cmidrule(lr){4-5}
\cmidrule(lr){6-7}
\cmidrule(lr){8-9}
\cmidrule(lr){10-11}

\textbf{Training} &
\textbf{F1} & \textbf{Acc.} &
\textbf{F1} & \textbf{Acc.} &
\textbf{F1} & \textbf{Acc.} &
\textbf{F1} & \textbf{Acc.} &
\textbf{F1} & \textbf{Acc.} \\
\midrule

OpenTSLM Single
& 37.49 & 77.34
& 31.19 & 67.19
& 15.03 & 35.71
& 32.94 & \textbf{33.33}
& 29.16 & 53.39 \\

OpenTSLM Pooled
& \textbf{41.80} & \textbf{80.47}
& \textbf{34.21} & \textbf{73.44}
& \textbf{19.70} & \textbf{38.72}
& \textbf{37.41} & 32.82
& \textbf{33.28} & \textbf{56.36} \\

\midrule

SLIP Single
& 31.78 & 87.16
& \textbf{33.93} & \textbf{40.01}
& 9.98 & 31.07
& \textbf{20.51} & 26.92
& 24.05 & 46.29 \\

SLIP Pooled
& \textbf{42.27} & \textbf{87.87}
& 30.00 & 31.32
& \textbf{24.78} & \textbf{45.61}
& 20.38 & \textbf{29.75}
& \textbf{29.36} & \textbf{48.64} \\

\bottomrule
\end{tabular}
}
\end{table*}

%% file: tab_read_performance.tex
% Remote I/O performance. PR #339 campaign, 2026-09-22.
% Harness eee7108f, clean tree, single revision across all 24 cells.
% Machine: AWS m6idn.2xlarge, 8 vCPU, 30 GiB RAM, local NVMe (ext4). One repetition per cell.
% TimeF rows: PR #339 artifacts read over presigned S3 GETs, hybrid transport
%   (whole-object GET for control.duckdb, Range reads for the values plane).
% Original rows: same transport as the published campaign (whole-object GETs, default readahead).
\begin{table*}[h!t]
\caption{Results of I/O performance assessment across 6 datasets.}
\label{tab:results:io_performance}
\centering
\footnotesize
\resizebox{\textwidth}{!}{%
\begin{tabular}{lllrrrrrrrr}
\toprule
\textbf{Dataset} &
\textbf{Representation} &
\textbf{Consumer} &
\shortstack{\textbf{Canonical item}\\\textbf{count [items]}} &
\shortstack{\textbf{Time to}\\\textbf{first item [s]}} &
\shortstack{\textbf{Full-read}\\\textbf{time [s]}} &
\shortstack{\textbf{Sequential}\\\textbf{throughput [items/s]}} &
\shortstack{\textbf{Decoded payload}\\\textbf{throughput [MB/s]}} &
\shortstack{\textbf{Transferred HTTP}\\\textbf{payload [GB]}} &
\shortstack{\textbf{HTTP request}\\\textbf{count [requests]}} &
\shortstack{\textbf{Peak resident}\\\textbf{set size [MiB]}} \\
\midrule

\multirow{4}{*}{HEARTS}
& \multirow{2}{*}{Original}
& pandas & 1,005 & \textbf{0.275} & 93.50 & 10.7 & 9.93 & 1.264 & 1,008 & \textbf{397.3} \\
& & PyTorch & 1,005 & \textbf{0.247} & 95.30 & 10.5 & 9.74 & 1.264 & 1,008 & \textbf{395.5} \\
& \multirow{2}{*}{\textit{TimeF}}
& \textit{pandas} & \textit{1,005} & \textit{0.574} & \textbf{\textit{18.43}} & \textbf{\textit{54.5}} & \textbf{\textit{50.38}} & \textbf{\textit{0.331}} & \textbf{\textit{220}} & \textit{879.1} \\
& & \textit{PyTorch} & \textit{1,005} & \textit{0.488} & \textbf{\textit{17.35}} & \textbf{\textit{57.9}} & \textbf{\textit{53.50}} & \textbf{\textit{0.331}} & \textbf{\textit{220}} & \textit{888.2} \\
\midrule

\multirow{4}{*}{ARFBench}
& \multirow{2}{*}{Original}
& pandas & 750 & \textbf{0.479} & 64.70 & 11.6 & 4.42 & 0.227 & 1,010 & 860.5 \\
& & PyTorch & 750 & \textbf{0.366} & 59.51 & 12.6 & 4.70 & 0.227 & 1,010 & 866.7 \\
& \multirow{2}{*}{\textit{TimeF}}
& \textit{pandas} & \textit{750} & \textit{0.833} & \textbf{\textit{11.40}} & \textbf{\textit{65.8}} & \textbf{\textit{25.06}} & \textbf{\textit{0.062}} & \textbf{\textit{131}} & \textbf{\textit{808.6}} \\
& & \textit{PyTorch} & \textit{750} & \textit{0.749} & \textbf{\textit{10.60}} & \textbf{\textit{70.7}} & \textbf{\textit{26.38}} & \textbf{\textit{0.062}} & \textbf{\textit{131}} & \textbf{\textit{827.8}} \\
\midrule

\multirow{4}{*}{VerbalTS}
& \multirow{2}{*}{Original}
& pandas & 104,628 & 3.407 & \textbf{53.74} & \textbf{1,946.8} & \textbf{11.05} & 0.675 & 146 & \textbf{958.5} \\
& & PyTorch & 104,628 & 2.934 & \textbf{15.15} & \textbf{6,907.4} & \textbf{37.01} & 0.674 & \textbf{140} & \textbf{956.2} \\
& \multirow{2}{*}{\textit{TimeF}}
& \textit{pandas} & \textit{104,628} & \textbf{\textit{1.087}} & \textit{70.68} & \textit{1,480.3} & \textit{8.40} & \textbf{\textit{0.431}} & \textbf{\textit{144}} & \textit{1,036.8} \\
& & \textit{PyTorch} & \textit{104,628} & \textbf{\textit{1.129}} & \textit{37.04} & \textit{2,824.8} & \textit{15.13} & \textbf{\textit{0.431}} & \textit{144} & \textit{1,011.0} \\
\midrule

\multirow{4}{*}{SLIP}
& \multirow{2}{*}{Original}
& pandas & 618,508 & \textbf{2.841} & \textbf{311.54} & \textbf{1,985.3} & \textbf{21.94} & 4.991 & \textbf{217} & 2,226.1 \\
& & PyTorch & 618,508 & \textbf{2.415} & \textbf{116.47} & \textbf{5,310.6} & \textbf{57.86} & 4.991 & \textbf{217} & 2,457.8 \\
& \multirow{2}{*}{\textit{TimeF}}
& \textit{pandas} & \textit{618,508} & \textit{9.764} & \textit{469.03} & \textit{1,318.7} & \textit{14.57} & \textbf{\textit{3.984}} & \textit{1,656} & \textbf{\textit{2,036.1}} \\
& & \textit{PyTorch} & \textit{618,508} & \textit{9.830} & \textit{282.78} & \textit{2,187.3} & \textit{23.83} & \textbf{\textit{3.984}} & \textit{1,656} & \textbf{\textit{2,033.9}} \\
\midrule

\multirow{4}{*}{ECG-QA}
& \multirow{2}{*}{Original}
& pandas & 231,543 & 28.787 & 2,735.10 & 84.7 & 20.53 & 3.267 & 32,305 & 2,139.6 \\
& & PyTorch & 231,543 & 32.103 & 2,351.40 & 98.5 & 23.80 & 3.267 & 32,305 & 2,139.0 \\
& \multirow{2}{*}{\textit{TimeF}}
& \textit{pandas} & \textit{231,543} & \textbf{\textit{1.740}} & \textbf{\textit{940.21}} & \textbf{\textit{246.3}} & \textbf{\textit{59.72}} & \textbf{\textit{2.048}} & \textbf{\textit{2,070}} & \textbf{\textit{1,150.5}} \\
& & \textit{PyTorch} & \textit{231,543} & \textbf{\textit{1.837}} & \textbf{\textit{644.28}} & \textbf{\textit{359.4}} & \textbf{\textit{86.88}} & \textbf{\textit{2.048}} & \textbf{\textit{2,070}} & \textbf{\textit{1,168.7}} \\
\midrule

\multirow{4}{*}{Sleep-EDF}
& \multirow{2}{*}{Original}
& pandas & 458,161 & \textbf{0.718} & \textbf{381.05} & \textbf{1,202.4} & \textbf{45.91} & 8.714 & \textbf{197} & \textbf{786.8} \\
& & PyTorch & 458,161 & \textbf{0.751} & \textbf{170.83} & \textbf{2,682.0} & \textbf{100.91} & 8.714 & \textbf{197} & \textbf{760.0} \\
& \multirow{2}{*}{\textit{TimeF}}
& \textit{pandas} & \textit{458,161} & \textit{1.767} & \textit{633.64} & \textit{723.1} & \textit{27.61} & \textbf{\textit{5.616}} & \textit{4,204} & \textit{1,081.8} \\
& & \textit{PyTorch} & \textit{458,161} & \textit{1.778} & \textit{431.64} & \textit{1,061.4} & \textit{39.94} & \textbf{\textit{5.616}} & \textit{4,204} & \textit{1,089.4} \\
\midrule

\bottomrule
\end{tabular}%
}
\vspace{-1em}
\end{table*}

%% file: 05_discussion.tex
\section{Discussion}
\paragraph{Principal findings.}
TimeNet provides a common representation and infrastructure for heterogeneous temporal datasets that differ in domain, modality, sampling structure, and task formulation. Across six datasets comprising more than 1.5 million task instances, the same interface supports data access and model-specific collation without dataset-specific loading pipelines, while maintaining practical I/O performance relative to native formats. This shared representation also enables datasets to be combined directly for training: for OpenTSLM, pooled training increases F1 on all four datasets. SLIP similarly improves on Sleep-EDF and HEARTS, but shows a substantial decrease on \acrshort{ecgqa}. The slight performance drop on \acrshort{ecgqa} might be explained by SLIPs smaller model size (400M parameters) compared to OpenTSLM 1B parameters, giving it less capacity to learn across tasks.
\textbf{Practical implications.}
We believe the research field must move beyond fragmented, one-off dataset releases toward publishing datasets into shared registries, with common means to read, write, annotate, expand, and version data, as provided by TimeNet. 
Recordings should be stored once, while new annotations and tasks can be added over time without duplicating or restructuring the original data. Toward this end, we have begun by combining heterogeneous data into a public registry. \\
%\textbf{Comparison with prior work.}
%TimeNet complements existing efforts such as GluonTS~\citep{Alexandrov20jmlr}, LOTSA~\citep{Woo24icml}, GIFT-Eval~\citep{Aksu24arxiv}, and MEDS~\citep{Arnrich24iclrw}, which often standardize model interfaces and dataset access for tasks such as forecasting and classification, but generally treat each dataset as a self-contained resource and leave annotations and task definitions outside the data representation. TimeNet instead jointly represents recordings, annotations, and tasks in a shared, versioned format, enabling heterogeneous datasets and supervision types to be accessed through the same interface without repackaging the underlying data.
\textbf{Limitations and future work.}
While TimeNet's data representation and interfaces are generally model- and task- independent, our training experiments evaluate only selected models and tasks rather than the full range of \gls{tfm} families. Our current evaluation focuses on data access across six datasets and training across four datasets, and, hence, only evaluates a subset of dataset and task families supported by TimeNet. Additionally, we only assesses training on a limited number of task types: Question answering and time series-based reasoning, but leave training of a model that spans all tasks to future work. %Our goal ultimately, is to bring over 100 datasets into the public registry and demonstrate improved model capabilities when training on this pool of data. 

%% file: affiliations.tex
% Variant B: full-width affiliations map with the institution names.
% Included from main_variantB.tex (after the Use-of-AI statement, before the
% references). Regenerate the map with: python tools/affiliations_map_labeled.py
\subsubsection*{Contributing Institutions}
\newcommand{\affl}[2]{\par\hangindent=1.5em\hangafter=1\noindent\makebox[1.5em][l]{\textcolor{tslmAccentDeep}{\sffamily\fontseries{sb}\selectfont#1}}#2\par}
\begin{figure}[H]
  \centering
  \includegraphics[width=\linewidth]{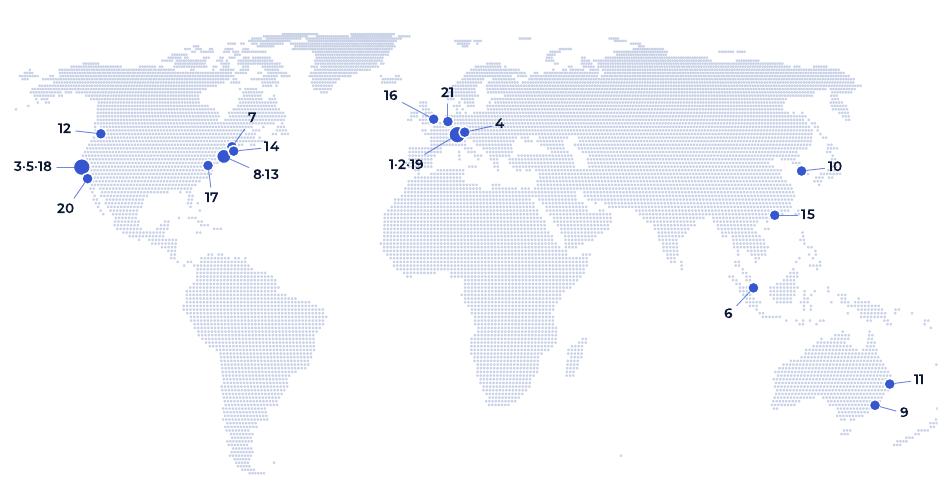}
  \vspace{6pt}
  {\fontsize{7.4}{10}\selectfont\color{tslmInk}\raggedright
  \begin{minipage}[t]{0.34\linewidth}
    \affl{1}{Aionic Labs, Switzerland}
    \affl{2}{Agentic Systems Lab, ETH Zurich, Switzerland}
    \affl{3}{Stanford University, USA}
    \affl{4}{Technical University of Munich, Germany}
    \affl{5}{Google}
    \affl{6}{National University of Singapore, Singapore}
    \affl{7}{Dartmouth College, USA}
  \end{minipage}\hfill
  \begin{minipage}[t]{0.28\linewidth}
    \affl{8}{New York University, USA}
    \affl{9}{UNSW Sydney, Australia}
    \affl{10}{KAIST, South Korea}
    \affl{11}{Griffith University, Australia}
    \affl{12}{Seattle Children's Hospital, USA}
    \affl{13}{Columbia University, USA}
    \affl{14}{Harvard University, USA}
  \end{minipage}\hfill
  \begin{minipage}[t]{0.34\linewidth}
    \affl{15}{Sun Yat-sen University, China}
    \affl{16}{University of Cambridge, UK}
    \affl{17}{University of Virginia, USA}
    \affl{18}{TimeCopilot, USA}
    \affl{19}{Centre for Digital Health Interventions, St.\ Gallen, Switzerland}
    \affl{20}{University of California, Los Angeles, USA}
    \affl{21}{Eindhoven University of Technology, The Netherlands}
  \end{minipage}\par}
  \caption*{\textbf{Contributing institutions.} Numbers follow the affiliation list on the first page.}
\end{figure}

%% file: 06_appendix.tex
\newpage
\section{TimeNet \acrshort{api} and Implementation Details}

\subsection{\acrshort{api} Example}\label{sec:app:implementation}
The TimeNet \acrshort{api} exposes the logical object hierarchy directly and hides the physical layout of the underlying TimeF version.
A consumer addresses any dataset by an \texttt{org/name[@version]} reference, while the client resolves its registry, reconstructs its typed objects from the relational control plane, and leaves Signal values lazy until they are read.
The same code therefore searches, loads, and traverses datasets independently of original file formats and source-specific dependencies.

The TimeNet \acrshort{api} is declarative and object-oriented.
Dataset authors are encouraged to use the object hierarchy directly, rather than providing a separate serialization interface.
As discussed in the main paper extensively, a Signal combines values with a typed specification and axis, Sources organize Signals recursively, and Records collect the resulting acquisition hierarchy.
Calling \texttt{annotate()} attaches reusable annotation content and returns the particular occurrence that a Task can reference as context or supervision.
Concrete Task classes then reference ordinary Record and Annotation objects, while \texttt{add\_record()} and \texttt{add\_task()} validate ownership, identity, temporal bounds, and cross-object references before modifying the Dataset.
Finally, a single \texttt{write()} call derives the schema and commits the manifest, DuckDB control plane, and Parquet or Zarr values plane as one immutable version.
Users consequently work with a small set of composable Python objects and are not required to interact with the relational layer or data storage directly.
Figure~\ref{fig:timenet_api} illustrates the ease of both consumption and authoring through the current public \acrshort{api}.

\input{api_example.tex}

\subsection{Datasets}
\label{sec:appendix:datasets}
\begin{table}[H]
\centering
\caption{The datasets, as released. Task types name the tags of TimeF's task
registry.}
\label{tab:datasets}
\footnotesize
\setlength{\tabcolsep}{3pt}
\renewcommand{\arraystretch}{1.15}

% Horizontally and vertically centered fixed-width column
\newcolumntype{C}[1]{>{\centering\arraybackslash}m{#1}}

\begin{tabular}{C{48pt}C{38pt}C{70pt}C{74pt}C{76pt}C{38pt}}
\toprule
Dataset & Size & Released format & Domains & Task types & Overlap \\
\midrule
Sleep-EDF Exp.\ & 8.1\,GB & EDF, EDF+ (int16), XLS &
  sleep, EEG, EOG, EMG, respiration &
  scalar prediction, temporal localization &
  SLIP \\

ECG-QA CoT & 1.89\,GB & WFDB (int16), CSV &
  cardiology, ECG &
  question answering &
  SLIP \\

HEARTS & 1.32\,GB & pickle (pandas DataFrame, float64) &
  metabolic, respiration, COVID cough &
  scalar prediction, forecasting, temporal localization &
  Capture-24 \\

ARFBench & 103\,MB &
  Parquet (datetime64, float64, string, int64), CSV &
  infrastructure, networking, database, security &
  & none \\

VerbalTS & 800\,MB &
  NPY (float64, UTF-32, int64), JSON &
  synthetic, energy, traffic, motion, weather &
  series generation &
  ETTm1, Weather \\

SLIP & 6.14\,GB &
  Parquet (nested float32, string) &
  health, web, nature, energy, IoT, transport, activity, stress &
  question answering &
  Sleep-EDF, PTB-XL, Capture-24 \\
\bottomrule
\end{tabular}
\end{table}

% \section{Preprocessing assumptions}
% \label{app:assumptions}

% Conversion is not neutral. Each release leaves facts underdetermined, and a
% connector that reads it has to decide. This appendix states the decisions that
% affect what the stored artifact contains, and therefore what
% Table~\ref{tab:main} measures.

% \paragraph{Sleep-EDF Expanded}
% [--]

% \paragraph{ECG-QA CoT.}
% [--]

% \paragraph{TSQA.}
% [--]

% \paragraph{HEARTS.}
% [--]

% \paragraph{ARFBench.}
% [--]

% \paragraph{VerbalTS.}
% [--]

% \paragraph{SLIP.}
% [--]

% ============================================================
% Required packages
% ============================================================
% \usepackage{amsmath,amssymb}
% \usepackage{booktabs}
% \usepackage{tabularx}
% \usepackage{enumitem}
% \usepackage[nameinlink,capitalise]{cleveref}

% Optional notation
% \newcommand{\X}{\mathbf{X}}
% \newcommand{\Y}{\mathbf{Y}}
% \newcommand{\C}{\mathbf{C}}
% \newcommand{\T}{\mathcal{T}}
% \newcommand{\task}{\tau}

% ============================================================

%% file: api_example.tex
\begin{figure}[h!t]
  \noindent
  \begin{minipage}[t]{0.32\linewidth}
  \vspace{0pt}
  \begin{apipanel}{Signals and spans}
  \input{listings/api_signals}
  \end{apipanel}
  \end{minipage}%
  \hfill
  \begin{minipage}[t]{0.32\linewidth}
  \vspace{0pt}
  \begin{apipanel}{Annotations}
  \input{listings/api_annotations}
  \end{apipanel}
  \end{minipage}%
  \hfill
  \begin{minipage}[t]{0.32\linewidth}
  \vspace{0pt}
  \begin{apipanel}{Tasks}
  \input{listings/api_tasks}
  \end{apipanel}
  \end{minipage}

  \par\vspace{2pt}
  \caption{TimeNet API: loading signals and defining spans, attaching annotations, and deriving tasks.}
  \label{fig:timenet_api}
  \end{figure}
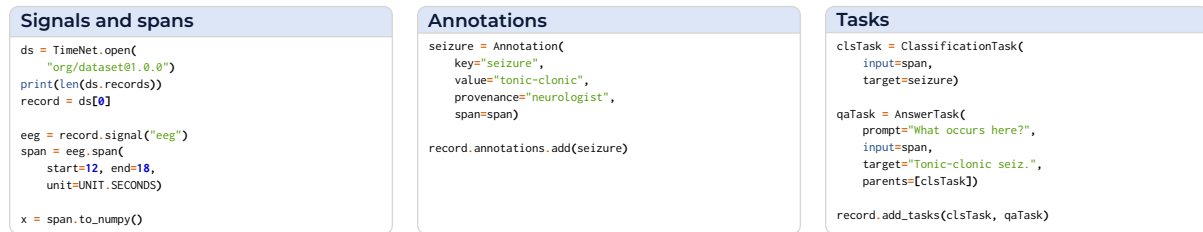

%% file: listings/api_signals.tex
\begin{Verbatim}[commandchars=\\\{\},fontsize=\tiny]
\PY{n}{ds} \PY{o}{=} \PY{n}{TimeNet}\PY{o}{.}\PY{n}{open}\PY{p}{(}
    \PY{l+s+s2}{\PYZdq{}}\PY{l+s+s2}{org/dataset@1.0.0}\PY{l+s+s2}{\PYZdq{}}\PY{p}{)}
\PY{n+nb}{print}\PY{p}{(}\PY{n+nb}{len}\PY{p}{(}\PY{n}{ds}\PY{o}{.}\PY{n}{records}\PY{p}{)}\PY{p}{)}
\PY{n}{record} \PY{o}{=} \PY{n}{ds}\PY{p}{[}\PY{l+m+mi}{0}\PY{p}{]}

\PY{n}{eeg} \PY{o}{=} \PY{n}{record}\PY{o}{.}\PY{n}{signal}\PY{p}{(}\PY{l+s+s2}{\PYZdq{}}\PY{l+s+s2}{eeg}\PY{l+s+s2}{\PYZdq{}}\PY{p}{)}
\PY{n}{span} \PY{o}{=} \PY{n}{eeg}\PY{o}{.}\PY{n}{span}\PY{p}{(}
    \PY{n}{start}\PY{o}{=}\PY{l+m+mi}{12}\PY{p}{,} \PY{n}{end}\PY{o}{=}\PY{l+m+mi}{18}\PY{p}{,}
    \PY{n}{unit}\PY{o}{=}\PY{n}{UNIT}\PY{o}{.}\PY{n}{SECONDS}\PY{p}{)}

\PY{n}{x} \PY{o}{=} \PY{n}{span}\PY{o}{.}\PY{n}{to\PYZus{}numpy}\PY{p}{(}\PY{p}{)}
\end{Verbatim}

%% file: listings/api_annotations.tex
\begin{Verbatim}[commandchars=\\\{\},fontsize=\tiny]
\PY{n}{seizure} \PY{o}{=} \PY{n}{Annotation}\PY{p}{(}
    \PY{n}{key}\PY{o}{=}\PY{l+s+s2}{\PYZdq{}}\PY{l+s+s2}{seizure}\PY{l+s+s2}{\PYZdq{}}\PY{p}{,}
    \PY{n}{value}\PY{o}{=}\PY{l+s+s2}{\PYZdq{}}\PY{l+s+s2}{tonic\PYZhy{}clonic}\PY{l+s+s2}{\PYZdq{}}\PY{p}{,}
    \PY{n}{provenance}\PY{o}{=}\PY{l+s+s2}{\PYZdq{}}\PY{l+s+s2}{neurologist}\PY{l+s+s2}{\PYZdq{}}\PY{p}{,}
    \PY{n}{span}\PY{o}{=}\PY{n}{span}\PY{p}{)}

\PY{n}{record}\PY{o}{.}\PY{n}{annotations}\PY{o}{.}\PY{n}{add}\PY{p}{(}\PY{n}{seizure}\PY{p}{)}
\end{Verbatim}

%% file: listings/api_tasks.tex
\begin{Verbatim}[commandchars=\\\{\},fontsize=\tiny]
\PY{n}{clsTask} \PY{o}{=} \PY{n}{ClassificationTask}\PY{p}{(}
    \PY{n+nb}{input}\PY{o}{=}\PY{n}{span}\PY{p}{,}
    \PY{n}{target}\PY{o}{=}\PY{n}{seizure}\PY{p}{)}

\PY{n}{qaTask} \PY{o}{=} \PY{n}{AnswerTask}\PY{p}{(}
    \PY{n}{prompt}\PY{o}{=}\PY{l+s+s2}{\PYZdq{}}\PY{l+s+s2}{What occurs here?}\PY{l+s+s2}{\PYZdq{}}\PY{p}{,}
    \PY{n+nb}{input}\PY{o}{=}\PY{n}{span}\PY{p}{,}
    \PY{n}{target}\PY{o}{=}\PY{l+s+s2}{\PYZdq{}}\PY{l+s+s2}{Tonic\PYZhy{}clonic seiz.}\PY{l+s+s2}{\PYZdq{}}\PY{p}{,}
    \PY{n}{parents}\PY{o}{=}\PY{p}{[}\PY{n}{clsTask}\PY{p}{]}\PY{p}{)}

\PY{n}{record}\PY{o}{.}\PY{n}{add\PYZus{}tasks}\PY{p}{(}\PY{n}{clsTask}\PY{p}{,} \PY{n}{qaTask}\PY{p}{)}
\end{Verbatim}

%% file: 07_appendix_task_taxonomy.tex
\section{A Compositional Taxonomy of Temporal Tasks}
\label{app:taxonomy}

This appendix gives the formal details behind
Section~\ref{sec:taskspec}. We first define the full task specification and
support relations, then enumerate the signature space and task families.
We subsequently describe the output-type registry, task specialization,
orthogonal interaction axes, task constructors, and the current scope of
the specification.

% ============================================================
\subsection{Full task specification}
\label{app:specification}

A record $r$ owns a recursive hierarchy of Sources whose leaf Signals form
the set $K_r$. Each signal
$x_k:\mathcal{T}_k\rightarrow\mathbb{R}^{c_k}$ is defined over a regular,
irregular, or ordinal temporal axis. Annotations $\mathcal{A}_r$ may be
associated with the Record, its Sources, or its Signals.

The set of channel--time cells available in a record is

\begin{equation}
    \Omega_r
    =
    \bigcup_{k\in K_r}
    \{k\}\times\mathcal{T}_k .
\end{equation}

A \emph{view} is any subset $W\subseteq\Omega_r$. Views may span multiple
signals or, for tasks involving several temporal objects, multiple records.

\begin{definition}[Task specification]
\label{def:task}
A temporal task is represented analytically as

\begin{equation}
\tau=
\big\langle
W_{\mathrm{in}},\,C,\,
(O,W_{\mathrm{out}},g),\,
\iota,\,\varepsilon,\,\pi,\,\mu
\big\rangle ,
\end{equation}

where:

\begin{itemize}[leftmargin=*,itemsep=1pt]
    \item $W_{\mathrm{in}}$ contains the input views over one or more records;
    \item $C$ contains optional conditioning such as metadata, text, a query,
          or in-context examples;
    \item $O$ is an output type from the extensible registry
          $\mathcal{O}$ (Section~\ref{app:types});
    \item $W_{\mathrm{out}}$ specifies the temporal support to which the
          target refers;
    \item $g$ is a registered deterministic function that derives the target
          from stored signals and annotations, with identity as the default;
    \item $\iota\in\{\text{structured},\text{question},\text{instruction}\}$
          specifies the interaction interface;
    \item $\varepsilon\in\{\text{none},\text{text},\text{evidence}\}$
          specifies the rationale requirement;
    \item $\pi\in\{\text{offline},\text{streaming}\}$ specifies the access
          protocol; and
    \item $\mu$ is a scoring function, either per instance
          $\mu(\hat y,y)$ or per corpus
          $\mu(\{\hat y_i\},\{y_i\})$.
\end{itemize}
\end{definition}

For non-signal targets, $W_{\mathrm{out}}$ denotes the temporal support
\emph{referenced} by the target rather than a physical location of the
output itself. For example, a classification label may describe an ECG
interval even though the label itself has no temporal extent.

The tuple above is an analytical representation rather than a literal
serialization of the TimeF Task object. Stored Tasks provide references to
Records, targets or target Annotations, and optional prompt, scope, rationale,
and provenance fields. Other components, such as the access protocol or
evaluation metric, may be supplied by the benchmark or model-facing layer.

% ============================================================
\subsection{Support relation and arity}
\label{app:support}

Let $K_{\mathrm{out}}$ denote the target channels and let

\begin{equation}
    W_{\mathrm{in}}^\ast
    =
    W_{\mathrm{in}}
    \cap
    \left(
        \bigcup_{k\in K_{\mathrm{out}}}
        \{k\}\times\mathcal{T}_k
    \right)
\end{equation}

be the observed portion of those channels.

\begin{definition}[Support relation and arity]
\label{def:support}
The target position is

\begin{equation}
\Delta(\tau)
\in
\{
\textsc{Obs},
\textsc{Gap},
\textsc{Fut},
\textsc{Chan},
\textsc{Null}
\},
\end{equation}

with:

\begin{itemize}[leftmargin=*,itemsep=1pt]
    \item \textsc{Obs}: the target refers to support exposed in
          $W_{\mathrm{in}}$;
    \item \textsc{Gap}: the target lies within the temporal extent of the
          observed target channels but is withheld from the input;
    \item \textsc{Fut}: the target lies after the observed support of the
          target channels;
    \item \textsc{Chan}: the target channels are not observed although other
          input channels are available; and
    \item \textsc{Null}: $W_{\mathrm{in}}=\emptyset$.
\end{itemize}

Independently,

\begin{equation}
    \alpha(\tau)\in\{1,\geq 2\}
\end{equation}

denotes the number of temporal objects over which the target is jointly
defined.
\end{definition}

Known future covariates do not violate
$\Delta=\textsc{Fut}$: the relation is computed with respect to the target
channels rather than all available channels.

Because the exposed input is represented explicitly by $W_{\mathrm{in}}$,
targets outside the observed support can be withheld mechanically. Under
$\pi=\text{streaming}$, the loader additionally exposes only the prefix
available at each prediction time.

% ============================================================
\subsection{The signature grid}
\label{app:grid}

Let

\begin{equation}
    \kappa(O)
    \in
    \{
        \textsc{Value},
        \textsc{Span},
        \textsc{Signal},
        \textsc{Index}
    \}
\end{equation}

denote the semantic kind of the output type. \textsc{Value} includes labels,
real-valued outputs, and free text. Composite outputs take the most temporally
structured kind of their components, using the precedence

\begin{equation}
\textsc{Signal}
\succ
\textsc{Span}
\succ
\textsc{Index}
\succ
\textsc{Value}.
\end{equation}

The task signature is

\begin{equation}
    \phi(\tau)
    =
    \big(
        \kappa(O),
        \Delta(\tau),
        \alpha(\tau)
    \big).
\end{equation}

Let

\begin{equation}
    \Delta_{\mathrm{in}}
    =
    \{
        \textsc{Obs},
        \textsc{Gap},
        \textsc{Chan}
    \}.
\end{equation}

Table~\ref{tab:grid} enumerates the resulting signature space.

\begin{table}[h]
\centering
\scriptsize
\caption{%
Signature grid for temporal tasks. Rows denote output kind $\kappa$ and
columns denote support relation $\Delta$ for single-object tasks.
The final column contains relational tasks with $\alpha\geq2$.
[A], [L], [F], [R], [C], and [T] denote attribute inference,
temporal localization, future-directed inference, relational inference,
temporal construction, and temporal transformation, respectively.}
\label{tab:grid}
\setlength{\tabcolsep}{3pt}
\renewcommand{\arraystretch}{1.2}
\begin{tabularx}{\linewidth}{
    @{}
    l
    *{6}{>{\centering\arraybackslash}X}
    @{}
}
\toprule
&
\multicolumn{5}{c}{single object ($\alpha=1$)}
&
relational
\\
\cmidrule(lr){2-6}
\cmidrule(l){7-7}

&
\textsc{Obs}
&
\textsc{Gap}
&
\textsc{Fut}
&
\textsc{Chan}
&
\textsc{Null}
&
$\alpha\geq2$
\\

\midrule

\textsc{Value}
&
[A] classification, scalar, answer
&
[A] value over masked region
&
[F] event prediction
&
[A] cross-channel inference
&
---
&
[R] comparison, ordering
\\ \addlinespace[4pt]

\textsc{Span}
&
[L] localization, segmentation
&
[L] localization in masked region
&
[F] onset prediction
&
[L] cross-channel localization
&
---
&
[R] alignment
\\ \addlinespace[4pt]

\textsc{Signal}
&
[T] editing, denoising
&
[T] imputation
&
[F] forecasting
&
[T] virtual sensing
&
[C] generation
&
[R] signal relation
\\ \addlinespace[4pt]

\textsc{Index}
&
---
&
---
&
---
&
---
&
---
&
[R] retrieval
\\

\bottomrule
\end{tabularx}
\end{table}

Single-object \textsc{Value} or \textsc{Span} targets with
$\Delta=\textsc{Null}$ are undefined because there is no temporal object to
describe or localize. Likewise, an \textsc{Index} output requires a candidate
set and is therefore relational.

\begin{proposition}[Partition of the signature space]
\label{prop:partition}
Every well-posed signature belongs to exactly one of the six constructions
in Table~\ref{tab:taxonomy}.
\end{proposition}

\begin{proof}
If $\alpha\geq2$, the task is relational inference. Otherwise
$\alpha=1$. If $\Delta=\textsc{Fut}$, the task is future-directed inference
regardless of output kind. For the remaining support relations, the output
kind distinguishes \textsc{Value}, \textsc{Span}, and \textsc{Signal}
targets. Signal targets with $\Delta=\textsc{Null}$ are temporal construction;
signal targets with $\Delta\in\Delta_{\mathrm{in}}$ are temporal
transformation. The remaining valid cases are attribute inference and
temporal localization. Ill-posed cells are excluded by definition.
\end{proof}

The two signal-producing constructions---construction and
transformation---form the broader \emph{temporal synthesis} family used in
Fig.~\ref{fig:task-tree}. Thus the six constructions induce five semantic
families.

% ============================================================
\subsection{Formal definitions of the constructions}
\label{app:families}

Let $X$ denote the selected input views and $C$ their conditioning.

\paragraph{Attribute inference.}
For

\begin{equation}
    \phi(\tau)
    =
    (\textsc{Value},\Delta_{\mathrm{in}},1),
\end{equation}

the task asks for a value describing the recorded object or a supported
region,

\begin{equation}
    f(X,C)\rightarrow y.
\end{equation}

The codomain distinguishes common task types:
$y\in\{c_1,\ldots,c_K\}$ gives classification,
$y\in\mathbb{R}^d$ gives scalar prediction, and free text gives
\texttt{answer}, including captioning or open-ended interpretation.

\paragraph{Temporal localization.}
For

\begin{equation}
    \phi(\tau)
    =
    (\textsc{Span},\Delta_{\mathrm{in}},1),
\end{equation}

the target identifies where or when a phenomenon occurs,

\begin{equation}
    f(X,q)
    \rightarrow
    \{s_1,\ldots,s_n\},
    \qquad
    s_i\subseteq\mathcal{T},
\end{equation}

where $q$ may specify the phenomenon. Examples include event localization,
anomaly localization, change-point detection, and segmentation.

\paragraph{Future-directed inference.}
For

\begin{equation}
    \phi(\tau)
    =
    (\ast,\textsc{Fut},1),
\end{equation}

the target lies beyond the observed support of the target channels,

\begin{equation}
    f(X_{\leq t},C)
    \rightarrow
    Y_{>t}.
\end{equation}

Forecasting is the signal-valued case. Value-valued targets include
future-state and event prediction, while span-valued targets include onset
prediction and related time-to-event formulations.

\paragraph{Relational inference.}
For

\begin{equation}
    \phi(\tau)
    =
    (\ast,\ast,\geq2),
\end{equation}

the target is jointly defined over multiple temporal objects,

\begin{equation}
    f(X_i,X_j,C)\rightarrow r.
\end{equation}

Examples include comparison, ordering, alignment, similarity, correspondence,
and retrieval. Temporal ordering between intervals may use Allen's interval
relations~\citep{allen1983maintaining}.

\paragraph{Temporal construction.}
For

\begin{equation}
    \phi(\tau)
    =
    (\textsc{Signal},\textsc{Null},1),
\end{equation}

the task constructs a temporal signal without an observed temporal input,

\begin{equation}
    f(C)\rightarrow X'.
\end{equation}

This includes conditional and unconditional generation.

\paragraph{Temporal transformation.}
For

\begin{equation}
    \phi(\tau)
    =
    (\textsc{Signal},\Delta_{\mathrm{in}},1),
\end{equation}

the task produces a temporal signal from an existing recorded object,

\begin{equation}
    f(X,C)\rightarrow X'.
\end{equation}

Editing and denoising correspond to observed support, imputation to a gap,
and virtual sensing to an unobserved target channel. Together with temporal
construction, these tasks form temporal synthesis.

% ============================================================
\subsection{Output-type registry}
\label{app:types}

TimeNet separates the semantic kind $\kappa$ from the concrete output type
$O$. Table~\ref{tab:types} gives representative built-in output types and
default evaluation functions.

\begin{table}[h]
\centering
\small
\caption{%
Representative output types. $^{\dagger}$Requires model or human judgment
and is therefore not directly usable as a verifiable reward.}
\label{tab:types}
\begin{tabular}{@{}lll@{}}
\toprule
Type & Kind $\kappa$ & Default $\mu$ \\
\midrule
$\textsc{Label}_K$, $\textsc{MultiLabel}_K$
    & \textsc{Value}
    & accuracy, macro-F1 \\

$\mathbb{R}^d$
    & \textsc{Value}
    & MAE, MAPE \\

\textsc{Text}
    & \textsc{Value}
    & exact match; model judge$^{\dagger}$ \\

$\textsc{Spans}(\mathcal{T})$
    & \textsc{Span}
    & event-F1, IoU \\

$\textsc{Signal}(\mathcal{T}',c)$
    & \textsc{Signal}
    & MASE, DTW, task-specific metrics \\

$\textsc{Index}_N$, $\textsc{Ranking}_N$
    & \textsc{Index}
    & accuracy, MRR, Recall@$k$ \\

\midrule

$\textsc{Dist}(O)$
    & $\kappa(O)$
    & CRPS, NLL, calibration \\

$\textsc{Set}(O_1\times O_2)$
    & highest-precedence component
    & component-wise \\

\bottomrule
\end{tabular}
\end{table}

A new output type is registered together with a validator, semantic kind, and
compatible scoring functions. Composite outputs inherit the highest-precedence
kind of their components. For example, labeled temporal segments are of kind
\textsc{Span} even though each segment also carries a categorical label.

% ============================================================
\subsection{Task specialization}
\label{app:specialization}

The signature determines the semantic construction, but does not identify
every benchmark-specific task. Let $P(\tau)$ denote constraints on the
admissible inputs, conditioning, output, and temporal support of a task,
excluding the orthogonal interface, rationale, protocol, and scoring fields.

We define

\begin{equation}
    \tau_a\preceq\tau_b
    \quad\Longleftrightarrow\quad
    \phi(\tau_a)=\phi(\tau_b)
    \;\land\;
    P(\tau_b)\subseteq P(\tau_a).
\end{equation}

Thus $\tau_a$ specializes $\tau_b$ if both perform the same semantic operation
and $\tau_a$ adds further restrictions. Examples include

\begin{equation}
\text{ECG rhythm classification}
\preceq
\text{classification},
\end{equation}

and

\begin{equation}
\text{anomaly localization}
\preceq
\text{temporal localization}.
\end{equation}

Likewise, single-step forecasting specializes forecasting by restricting the
target to one future time point, and probabilistic forecasting specializes it
by constraining the target to a predictive distribution.

The stored \texttt{from\_tasks} relation may record such derivations as
provenance. It does not require the child Task to physically copy fields from
its antecedent or itself enforce the partial order.

% ============================================================
\subsection{Orthogonal axes: interface, reasoning, and protocol}
\label{app:orthogonal}

The fields $\iota$, $\varepsilon$, and $\pi$ do not alter
$\phi(\tau)$.

\paragraph{Interaction interface.}
Question answering is an interface rather than a semantic family. For example,

\begin{align}
\text{``What state is present?''}
&\mapsto
(\textsc{Value},\textsc{Obs},1),
\\
\text{``When does the event occur?''}
&\mapsto
(\textsc{Span},\textsc{Obs},1),
\\
\text{``What happens next?''}
&\mapsto
(\ast,\textsc{Fut},1),
\\
\text{``Which event occurred first?''}
&\mapsto
(\ast,\ast,\geq2).
\end{align}

The same semantic task can therefore be exposed through a structured output,
question, or instruction without changing its family.

\paragraph{Reasoning.}
For a base task

\begin{equation}
    f(X,C)\rightarrow Y,
\end{equation}

a rationale requirement may additionally request

\begin{equation}
    f(X,C)\rightarrow(Y,E),
\end{equation}

where $E$ contains evidence, intermediate conclusions, or an explanation.
This does not change the semantic signature of the underlying task.

Evidence can itself be constrained, for example to stored spans or
annotations, allowing explanation quality and target correctness to be
evaluated separately.

\paragraph{Protocol.}
Under $\pi=\text{streaming}$, predictions are made at times
$t_1<t_2<\cdots$, with each prediction receiving only the input prefix
available at that point. Online anomaly detection and early classification
are therefore protocol variants of localization and attribute inference
rather than separate semantic families.

% ============================================================
\subsection{A compositional task algebra}
\label{app:algebra}

Conceptually, a task factorizes as

\begin{equation}
\text{Task}
=
\underbrace{\phi(\tau)}_{\text{semantic construction}}
\times
\underbrace{P(\tau)}_{\text{specialization}}
\times
\underbrace{
    \iota\times\varepsilon\times\pi
}_{\text{presentation and protocol}}
\times
\underbrace{\mu}_{\text{scoring}} .
\end{equation}

The signature determines the semantic operation, $P(\tau)$ determines the
task specialization, the orthogonal axes determine how the task is exposed
and supervised, and $\mu$ determines how predictions are evaluated.

% ============================================================
\subsection{Task constructors}
\label{app:constructors}

The following constructors describe common transformations between task
specifications. They formalize relationships between tasks rather than
requiring corresponding public SDK operations.

\begin{definition}[Task constructors]
\label{def:constructors}
Given a valid task specification $\tau$:

\begin{itemize}[leftmargin=*,itemsep=1pt]

    \item $\textsc{Ask}_q$ changes the interface to a natural-language
          question;

    \item $\textsc{Explain}$ adds a rationale requirement;

    \item $\textsc{Stream}$ changes the access protocol to streaming;

    \item $\textsc{Shift}_h$ withholds the final $h$ units of target-channel
          support from the input, moving the target to
          $\Delta=\textsc{Fut}$;

    \item $\textsc{Mask}_M$ removes cells $M$ from the input and uses them as
          targets, producing \textsc{Gap} or \textsc{Chan} support;

    \item $\textsc{Reduce}_g$ maps a structured target to a value, e.g.,
          existence, count, duration, or summary statistics;

    \item $\textsc{Pair}_\rho$ combines temporal objects through a relation
          $\rho$, producing a relational task; and

    \item $\textsc{Restrict}$ introduces additional constraints while
          preserving the signature, i.e., specialization under $\preceq$.

\end{itemize}
\end{definition}

Thus \textsc{Shift}, \textsc{Mask}, \textsc{Reduce}, and \textsc{Pair}
may move a task between regions of the signature space;
\textsc{Restrict} moves within a region; and
\textsc{Ask}, \textsc{Explain}, and \textsc{Stream} alter only orthogonal
axes.

For example, a language-conditioned anomaly-localization task may be written
conceptually as

\begin{equation}
    \textsc{Ask}
    \big(
        \textsc{Restrict}
        (
            \texttt{temporal localization}
        )
    \big),
\end{equation}

while forecasting with an explanation is

\begin{equation}
    \textsc{Explain}
    (
        \texttt{forecasting}
    ).
\end{equation}

% ============================================================
\subsection{Constructor preconditions}
\label{app:preconditions}

Constructors are only valid when the supervision required by the resulting
task exists.

\begin{center}
\small
\begin{tabularx}{\linewidth}{@{}l l X@{}}
\toprule
Constructor & Effect & Precondition \\
\midrule

$\textsc{Ask}_q$
&
$\iota\leftarrow\text{question}$
&
The question determines the intended target unambiguously.
\\

$\textsc{Explain}$
&
$\varepsilon\leftarrow\text{text/evidence}$
&
A rationale target exists in stored annotations, or the rationale is
evaluation-only.
\\

$\textsc{Stream}$
&
$\pi\leftarrow\text{streaming}$
&
The task and metric are defined for predictions from prefixes.
\\

$\textsc{Shift}_h$
&
$\Delta\leftarrow\textsc{Fut}$
&
The target has temporal support and sufficient target-channel history remains
before the prediction boundary.
\\

$\textsc{Mask}_M$
&
$\Delta\leftarrow\textsc{Gap}/\textsc{Chan}$
&
$M$ contains cells for which ground truth is stored.
\\

$\textsc{Reduce}_g$
&
$\kappa\leftarrow\textsc{Value}$
&
$g$ is deterministic and its codomain is registered in $\mathcal{O}$.
\\

$\textsc{Pair}_\rho$
&
$\alpha\leftarrow\geq2$
&
$\rho$ is defined for the paired target types.
\\

$\textsc{Restrict}$
&
$\preceq$
&
Each restricted field is a subset of the corresponding parent constraint and
the resulting task has at least one valid instance.
\\

\bottomrule
\end{tabularx}
\end{center}

Type validity does not imply semantic usefulness. For example, mechanically
shifting a record-level diagnosis does not necessarily define a meaningful
prediction problem. Derived tasks should therefore retain their derivation as
provenance, and task constructors must satisfy domain-specific as well as
type-level validity conditions.

\begin{proposition}[Type closure]
\label{prop:closure}
If $\tau$ is valid and a constructor's preconditions hold, the constructed
task is also type-valid.
\end{proposition}

Consequently, a registry may expose many task views over the same underlying
recordings. This growth does not require duplicating the source signals and,
when the required supervision is already stored, does not require additional
labeling.

% ============================================================
\subsection{Scope and limitations}
\label{app:scope}

The current specification covers supervised and self-supervised mappings from
temporal views to typed targets. For example, masked reconstruction and
next-step prediction can be represented through \textsc{Mask} and
\textsc{Shift} applied to signal targets.

The following settings are not yet first-class:

\begin{enumerate}[leftmargin=*,itemsep=1pt]

    \item \emph{Interventional and counterfactual tasks.}
    Conditioning may describe an intervention, but the current support
    relation does not distinguish observational from interventional targets.

    \item \emph{Sequential decision-making.}
    Tasks whose outputs are actions that change subsequent observations
    require an interactive environment rather than a fixed stored target.

    \item \emph{Multi-turn and agentic interaction.}
    Dialogue over a record, tool use, and iterative observation/action loops
    are outside the current task specification.

    \item \emph{Structured outputs beyond the registry.}
    Examples include causal graphs over channels or hierarchical event
    structures. Such targets can be incorporated by registering new output
    types, but are not currently built in.

\end{enumerate}

The specification is therefore intended to be extensible rather than
exhaustive: new output types, relations, and task constraints can be added
without changing the underlying representation of temporal recordings.